\documentclass[twoside,11pt]{article}
\usepackage[abbrvbib, preprint]{jmlr2e}

\usepackage[utf8]{inputenc} %
\usepackage[T1]{fontenc}    %
\usepackage{hyperref}       %
\usepackage{url}            %
\usepackage{booktabs}       %
\usepackage{amsfonts}       %
\usepackage{amsmath}        %
\usepackage{amsthm}         %
\usepackage{amssymb}        %
\usepackage{enumitem}       %
\usepackage{nicefrac}       %
\usepackage{microtype}      %
\usepackage{xcolor}         %
\usepackage{graphicx}       %
\usepackage{subcaption}     %
\usepackage[most]{tcolorbox} %
\usepackage{wrapfig}
\usepackage{multirow}
\usepackage{caption}
\graphicspath{{figures/}}

\newtcolorbox{contribbox}[1][]{%
  enhanced,
  colback=blue!3,
  colframe=blue!55!black,
  boxrule=0.8pt,
  arc=2pt,
  left=6pt, right=6pt, top=4pt, bottom=4pt,
  fonttitle=\bfseries,
  coltitle=blue!55!black,
  colbacktitle=white,
  attach boxed title to top left={yshift=-2mm, xshift=4mm},
  boxed title style={sharp corners, boxrule=0.8pt, colframe=blue!55!black},
  title={Highlight},
  #1
}

\usepackage{thmtools}
\newtheorem{theorem}{Theorem}[section]
\newtheorem{proposition}[theorem]{Proposition}

\newtheorem{remark}[theorem]{Remark}

\title{Controlling Transient Amplification Improves\\ Long-horizon Rollouts}

\ShortHeadings{}{Controlling Transient Amplification Improves Long-horizon Rollouts}
\firstpageno{1}

\author{\name Adeel Pervez \email Adeel.Pervez@ist.ac.at \\
       \addr Institute of Science and Technology Austria\\
       Klosterneuburg, Austria 
       \AND
       \name Francesco Locatello \email Francesco.Locatello@ist.ac.at\\
       \addr Institute of Science and Technology Austria\\
       Klosterneuburg, Austria 
       }

\begin{document}
\maketitle
\begin{abstract}
  Autoregressive neural simulators now match classical solvers on short-horizon prediction of physical systems, yet their accuracy degrades rapidly when rolled out over long horizons. 
In this work, we identify \emph{transient amplification} of perturbations around rollout trajectories as a structural mechanism driving rollout error.
Using a linearization analysis we show that when the Jacobians along an autoregressive trajectory are non-normal and non-commuting, the model amplifies errors transiently, resulting in model rollout drift even when the overall system is asymptotically stable. 
Building on the analysis, we propose \emph{commutativity regularization}: a combination of two penalties designed to reduce the normality defect of individual Jacobians and the commutator norm of Jacobians across steps.
The penalties are estimated with Jacobian-vector products and have no inference-time cost.
We show a propagator bound that quantifies rollout error under approximate commutativity and normality.
We evaluate UNet and FNO variants with commutativity regularization on 1D and 2D spatio-temporal data in synthetic and real settings, showing successful long-horizon rollouts over thousands of steps. 
Further, we show that the method improves FourCastNet climate forecasts on ERA5 without using any new data.
The gain is most pronounced out-of-distribution: trained on trajectories of a few hundred steps, regularized models remain in-distribution for thousands of rollout steps on initial conditions where baselines diverge.
  
\end{abstract}

\section{Introduction}
\label{sec:intro}
Neural networks that simulate physical systems one step at a time have become remarkably
capable.
Models such as FourCastNet~\citep{pathak2022fourcastnet}, GraphCast~\citep{lam2023graphcast}, Pangu-Weather~\citep{bi2023pangu},     Aurora~\citep{bodnar2024aurora}, NeuralGCM~\citep{kochkov2024neuralgcm}  and the diffusion-based GenCast~\citep{price2025gencast} now match or        surpass operational numerical weather prediction at
lead times of five to seven days. 
Similar progress has been made across fluid dynamics~\citep{sanchezgonzalez2020gns},
climate emulation~\citep{kochkov2024neuralgcm}, and molecular simulation~\citep{batzner2022nequip} using autoregressive architectures such as Fourier Neural
Operators, U-Nets and graph neural networks~\citep{li2021fno,ronneberger2015unet}.
On the other hand, the same models, when run iteratively for many steps, degrade far faster than their
single-step accuracy would suggest: predictions that closely match the ground truth at one
step may diverge or drift within tens of steps on complex systems~\citep{lippe2023pderefiner}.
Bridging this gap between single-step and long-rollout performance is the central open
problem for neural simulators of physical dynamics.

The standard explanation of this phenomenon is \emph{distributional shift}.
During training the model sees clean ground-truth states as inputs; at rollout time it
receives its own imperfect predictions, which progressively drift away from the training
distribution.
Recent methods address this by exposing the model to its own errors during
training~\citep{lippe2023pderefiner,cao2025spectralrefiner} or by applying iterative
corrections at inference time~\citep{kohl2024diffusion}, requiring multi-step training
or expensive inference.
These approaches are often effective, but they treat rollout distribution shift as a data or
inference problem and do not identify the structural mechanisms underlying error growth.

\looseness=-1In this paper, we identify one such mechanism in \emph{transient amplification} of rollout error. 
For a model linearized around its predicted trajectory, the prediction error at each step evolves under a sequence of local linear Jacobian maps.
When such a map is \emph{normal} \footnote{A linear operator or matrix $A$ is \emph{normal} if $AA^\top = A^\top A$. Normal matrices have a full set of \emph{orthogonal} eigenvectors.} it acts cleanly on errors, stretching them along
orthogonal directions so that each direction behaves in isolation.
\emph{Non-normal} maps have \emph{non-orthogonal} directions of amplification which interact to produce error growth. 
For asymptotically stable maps these excursions of growth, called \emph{transients}, eventually die away.
Prior work has frequently targeted \emph{stability} as a desirable 
property that guarantees asymptotic decay of error.
Rollouts, however, are evaluated at tens to hundreds of steps, precisely the
medium-horizon regime where transients dominate: stable maps can still produce large transient amplification at the
horizons that matter in practice.

Normality of individual Jacobians, however, is only part of the story.
Nonlinear models have \emph{state-dependent} Jacobians for which 
the directions of error amplification may rotate with time.
When these directions are misaligned across steps, error components can amplify even when no individual step does so on its own.
Aligning successive directions can mitigate this amplification; and a sufficient condition for alignment is that Jacobians across steps be \emph{commutative}. 
Commuting normal maps share common orthogonal eigenbases, 
which suppresses transient amplification at every horizon.
Non-commuting, non-normal Jacobians violate both conditions simultaneously,
and unregularized
networks trained on non-normal dynamics tend to learn Jacobians with this combination of
properties.

To control transient amplification directly, we propose \emph{commutativity
regularization}, inclusive of two auxiliary training losses applied over latent space operators.
A \emph{commutator penalty} regulates commutativity of consecutive Jacobians,
targeting the cross-step compounding of transients; a \emph{normality penalty}
discourages individual Jacobians from generating large transients in the first place.
The two are complementary: penalizing non-commutativity alone still permits large within-step
transients, and penalizing non-normality alone leaves cross-step compounding
uncontrolled.
Their combination drives the learned dynamics toward a regime in which errors decay
at the rate set by the Jacobian spectra with controlled intermediate amplification.
At inference, the rollout procedure is unchanged and there is no additional cost.

We validate the method on the following settings:
the 1D KdV equation, the 2D Barotropic Vorticity Equation, sea surface temperature data,
and
fine-tuning of the pretrained FourCastNet global weather model on ERA5 atmospheric data.
In each setting, commutativity regularization improves long-horizon rollout accuracy without modifying the inference procedure.

\paragraph{Contributions.}
We make the following contributions.
\begin{enumerate}[label=(\roman*)]
  \item We identify non-normal, non-commuting latent Jacobians as a structural driver
    of transient amplification in autoregressive neural operators, one not addressed by stability.
  \item We propose commutativity regularization: a commutator penalty together with a
    normality penalty, both estimated via Jacobian-vector products with zero inference-time overhead.
  \item We show a propagator bound under approximate commutativity and normality that
    replaces the per-step spectral norm with the shared spectral radius as the effective
    decay rate, a strict improvement whenever the Jacobians are non-normal.
  \item We validate the method on settings spanning chaotic, geophysical,
    integrable, and synthetic dynamical systems as well as large-scale atmospheric
    forecasting, obtaining consistent long-horizon accuracy gains over thousands of steps, including in an OOD setting. 
\end{enumerate}

\label{sec:method}

\section{Autoregressive Neural Operators and Transient Error}
\label{sec:background:rollout}

Let $F_\theta : \mathcal{X} \to \mathcal{X}$ be an operator trained to advance a dynamical state by one timestep.
At inference the model is applied \emph{autoregressively}: given an initial condition
$x_0 \in \mathcal{X}$, the rollout is
\begin{equation}
  \hat{x}_0 = x_0, \qquad
  \hat{x}_{t+1} = F_\theta(\hat{x}_t), \quad t = 0,1,\ldots,T-1.
  \label{eq:rollout}
\end{equation}
The model is typically trained on a single-step mean-squared loss
$\mathcal{L}_\mathrm{pred} = \mathbb{E}_{t}\bigl[\|F_\theta(x_t) - x_{t+1}\|^2\bigr]$,
which does not expose it to its own prediction errors during training.

\paragraph{Error propagation.}
Let $\varepsilon_t = x_t - \hat{x}_t$ denote the prediction error at step $t$.
Linearising $F_\theta$ around the predicted trajectory gives
\begin{equation}
  \varepsilon_{t+1} \approx J_t\,\varepsilon_t,
  \qquad
  J_t \;=\; \frac{\partial F_\theta}{\partial x}\Bigr|_{\hat{x}_t},
  \label{eq:linearised-error}
\end{equation}
so that after $T$ steps $\varepsilon_T \approx \Phi_T\,\varepsilon_0$, where the
\emph{propagator} is the ordered Jacobian product
\begin{equation}
  \Phi_T \;=\; J_{T-1}\,J_{T-2}\,\cdots\,J_0.
  \label{eq:propagator}
\end{equation}
The rollout error is therefore bounded by
$\|\varepsilon_T\|_2 \le \|\Phi_T\|_2\,\|\varepsilon_0\|_2$.
\looseness=-1Controlling the spectral norm of individual Jacobians,
$\|J_t\|_2 \le \alpha < 1$, yields $\|\Phi_T\|_2 \le \alpha^T$
by submultiplicativity.
However, this estimate conceals a key problem:
for \emph{non-normal} Jacobians, the intermediate propagator norm $\text{sup}_{t\le T} \|\Phi_t\|_2$ can greatly exceed $ \|\Phi_T\|_2$
due to ordering effects between non-commuting factors. 
This intermediate growth can occur even when each individual factor is contractive and the system is asymptotically stable \citep{trefethen2005spectra}.
This is the transient amplification mechanism we target.

\begin{contribbox}
\vspace{2pt}
Suppressing transient amplification of error across a rollout improves long-horizon accuracy in dynamical models with no inference-time overhead.
\emph{Commutativity regularization} achieves this by encouraging normality and commutativity of Jacobians around rollout trajectories.
\end{contribbox}

\subsection{Non-Normal Matrices and Transient Growth}
\label{sec:background:nonnormal}

A matrix $A$ is \emph{normal} if $A^\top A = A A^\top$.
For normal matrices the eigenvectors are orthonormal, the spectral radius
equals the spectral norm, $\rho(A) = \|A\|_2$, and there is no transient
growth: $\|A^t\|_2 = \rho(A)^t$ for all $t \ge 0$.

\begin{figure}[t]%
  \vspace{-1em}
  \centering
  \includegraphics[width=0.85\textwidth]{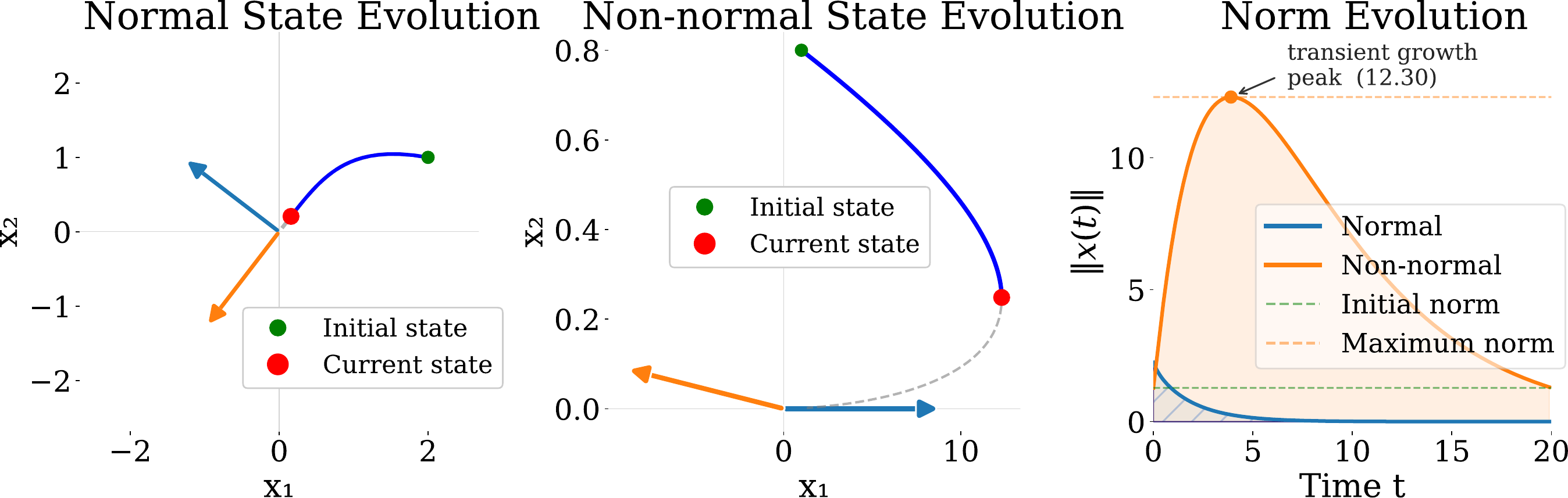}%
  \caption{Normal vs.\ non-normal transient growth on $2\times2$ stable systems,
           \emph{Left:} normal system, eigenvectors are
           orthogonal, %
           \emph{Centre:} non-normal system, eigenvectors are nearly
           parallel, %
            \emph{Right:} norm $\|x(t)\|$ showing transient growth. 
           }
  \label{fig:nonnormal}
\end{figure}%

For a non-normal matrix, by contrast, it may be that $\rho(A) \ll \|A\|_2$ and $\|A^t\|_2$ can vastly exceed
$\rho(A)^t$ for intermediate $t$ (transient growth), even though asymptotically $\lim_{t \rightarrow \infty} \|A^t\|^{1/t}_2 = \rho(A)$ \citep{horn1985}. 
The same phenomenon arises in products of \emph{distinct} matrices.
For $A_0,\ldots,A_{T-1}$ each satisfying $\|A_t\|_2 \le \alpha$, submultiplicativity gives
\begin{equation}
  \|A_{T-1}\cdots A_0\|_2 \;\le\; \alpha^T.
  \label{eq:submult}
\end{equation}
When the matrices are non-commuting, and are not simultaneously diagonalizable, the bound can be highly non-tight, and the product can exhibit significant transient amplification relative to its asymptotic behavior, even when all $\rho(A_t) < 1$.

\paragraph{Illustration.}
Figure~\ref{fig:nonnormal} contrasts the two behaviours on $2\times2$ toy systems.
Both matrices have all eigenvalues with negative real part, so both systems are
asymptotically stable.
The \emph{normal} system  has orthogonal
eigenvectors: the state decays smoothly along the eigendirections and the
norm $\|x(t)\|$ decreases monotonically.
The \emph{non-normal} system has nearly
parallel eigenvectors: the initial state contains large, nearly-cancelling
projections onto them.
During the transient phase these components briefly reinforce one another and
the norm \emph{grows}, reaching a peak well above the initial value, before
eventually decaying.
This transient amplification is invisible to a purely eigenvalue-based stability
analysis \citep{trefethen2005spectra}, yet it is the mechanism that compounds across time steps in
autoregressive rollout for neural operators trained on spatio-temporal data.

\subsection{Temporal Commutativity as the LTV Analogue of Normality}
\label{sec:background:ltv}

The system $x_{t+1} = J_t\,x_t$ is a linear time-varying (LTV) system \citep{rugh1981nonlinear}.
For linear time-invariant (LTI) systems (static $J_t = A$), normality ensures $\|A^T\|_2 = \rho(A)^T$,
providing tight control of long-horizon amplification.
For LTV systems the direct analogue is \emph{temporal commutativity} in addition to normality:

\begin{table}[t]%
  \centering
  \caption{LTI normality concepts and  LTV analogues.}
           
  \label{tab:ltv_normality}
  \small
  
  \begin{tabular}{ll}
    \toprule
    LTI (static $A$) & LTV (time-varying $\{J_t\}$) \\
    \midrule
    $A^\top A = A A^\top$ (normality) &
      $[J_{t_1}, J_{t_2}] = 0\;\forall\,t_1,t_2$ (normality and commutativity) \\
    Orthonormal eigenvectors &
      Shared orthonormal eigenbasis \\
    $\|A^T\|_2 = \rho(A)^T$ &
      $\|\Phi_T\|_2 \le \rho^T \le \alpha^T$ \\
    No transient growth &
      No ordering-induced transients \\
    \bottomrule
  \end{tabular}%
  \vspace{1mm}
\end{table}
if $[J_{t_1}, J_{t_2}] \equiv J_{t_1} J_{t_2} - J_{t_2} J_{t_1} = 0$ for all $t_1, t_2$,
then the matrices share a common eigenbasis~\citep{horn1985} and the propagator satisfies
\begin{equation}
  \Phi_T \;=\; J_{T-1}\cdots J_0 \;=\; U\,\Lambda_{T-1}\cdots\Lambda_0\,U^\top,
  \label{eq:commuting-product}
\end{equation}
where $U$ is the shared eigenbasis.
Consequently $\|\Phi_T\|_2 \le \rho^T$,
where $\rho = \max_{i,t}|\lambda_i(J_t)| \le \alpha$.
This recovers a tight LTI-like bound, with the spectral radius $\rho$
(which may be strictly less than the spectral norm $\alpha$ for non-normal matrices)
controlling growth rather than the norm.
The correspondence between LTI normality and LTV temporal commutativity
is in Table~\ref{tab:ltv_normality}.

\subsection{Error Bound Under Approximate Commutativity and Normality}
\label{sec:theory}

We now state the key theoretical result motivating our regularizer.
The ideal case is when all Jacobians are simultaneously, orthogonally diagonalizable.
The following theorem characterises how far the propagator deviates from this ideal
when the two conditions, all-pairs commutativity and
per-step normality, hold approximately.
The proof is deferred to Appendix \ref{app:method:prop-proof}.

\begin{restatable}[Propagator bound under approximate commutativity and normality]{theorem}{proptheorem}\label{thm:propagator}
  Let $J_0, J_1, \ldots, J_{T-1} \in \mathbb{R}^{n \times n}$ satisfy
  \begin{enumerate}[label=(\roman*)]
    \item $\|J_t\|_2 \le \alpha < 1$ \emph{(contraction)} for all $t$,
    \item $\rho(J_t) \le \rho \le \alpha$ \emph{(spectral radius bound)} for all $t$,
    \item $\|[J_s, J_t]\|_F \le \varepsilon$ for all $0 \le s, t \le T-1$
          \emph{(all-pairs approximate commutativity)}, and
    \item $\|J_t^\top J_t - J_t J_t^\top\|_F \le \eta$ \emph{(approximate normality)}
          for all $t$.
  \end{enumerate}
  Then the propagator satisfies
  $  \|\Phi_T\|_2 \;\le\; \rho^T \;+\; 2T\alpha^{T-1}\,\delta(\varepsilon,\eta),
  $
  where $\delta(\varepsilon,\eta) \to 0$ as $\varepsilon,\eta \to 0$.
\end{restatable}

\subsection{Commutativity and Normality Regularization}
\label{sec:method:comm-reg}

Theorem~\ref{thm:propagator} provides two auxiliary losses that penalize the mechanisms driving the growth of $\|\Phi_t\|_2$.

\begin{figure}[t]%
  \centering
  \includegraphics[width=0.8\textwidth]{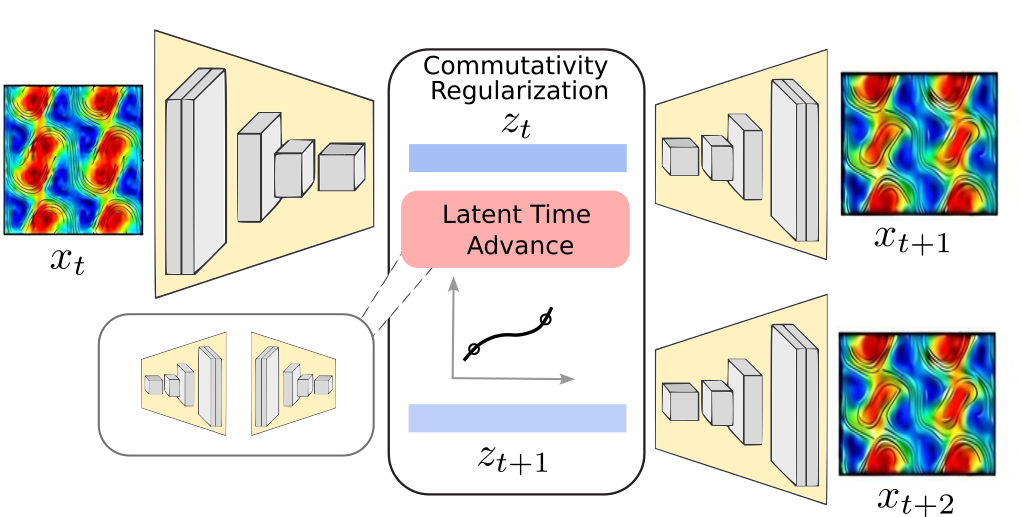}
  \caption{Latent advance architecture used by commutativity regularization. Other configurations possible (See Appendix~\ref{app:reg-config}.}
  \label{fig:reg-arch-latent}
\end{figure}

\paragraph{Latent-space Jacobians.}
For a given hidden representation $z_t$ of input $x_t$ for a neural model
define the \emph{latent advance map} $G_\theta : \mathbf{z}_t \mapsto \mathbf{z}_{t+1}$
as the bottleneck-to-bottleneck dynamics, and let
$\mathbf{J}_t = \partial G_\theta / \partial \mathbf{z}\big|_{\mathbf{z}_t}$
denote its Jacobian.
For a UNet we use the bottleneck feature, for FNO we use the middle layer features (Figure \ref{fig:reg-arch-latent}).
Regularizing latent map Jacobians prevents the model from introducing further non-normality in representation space.  
For large models (such as FourCastNet) we regularize a selected block as shown in Figure \ref{fig:reg-arch-block}.

\paragraph{Commutator penalty.}
We estimate the squared Frobenius norm of the commutator
$[\mathbf{J}_{t+1}, \mathbf{J}_t]$
using a single random probe $\mathbf{v} \sim \mathcal{N}(0,\,I)$ per minibatch
(Hutchinson estimator):
\begin{equation}
  \mathcal{L}_\mathrm{comm}
  \;=\;
  \mathbb{E}_{\mathbf{v}}\!\left[
    \bigl\|
      \mathbf{J}_{t+1}\,\mathbf{J}_t\,\mathbf{v}
      \;-\;
      \mathbf{J}_t\,\mathbf{J}_{t+1}\,\mathbf{v}
    \bigr\|^2
  \right].
  \label{eq:L_comm}
\end{equation}

\paragraph{Normality penalty.}
To suppress within-step transient amplification we additionally penalise
the normality defect of each $\mathbf{J}_t$:
\begin{equation}
  \mathcal{L}_\mathrm{norm}
  \;=\;
  \mathbb{E}_{\mathbf{v}}\!\left[
    \bigl\|
      \mathbf{J}_t^\top \mathbf{J}_t\,\mathbf{v}
      \;-\;
      \mathbf{J}_t\, \mathbf{J}_t^\top\,\mathbf{v}
    \bigr\|^2
  \right].
  \label{eq:L_norm}
\end{equation}
This penalises $\|J_t^\top J_t - J_t J_t^\top\|_F^2$ in expectation.
Together with $\mathcal{L}_\mathrm{comm}$, it drives $\mathbf{J}_t$ toward the ideal
case of Theorem~\ref{thm:propagator}: contractive, normal, and temporally commutative.

\paragraph{Training objective.}
The full training loss is
\begin{equation} 
 \mathcal{L}
  \;=\;
  \mathcal{L}_\mathrm{pred}
  \;+\;
  \lambda_\mathrm{c}\,\mathcal{L}_\mathrm{comm}
  \;+\;
  \lambda_\mathrm{n}\,\mathcal{L}_\mathrm{norm}, 
  \label{eq:total-loss}
\end{equation}
where $\mathcal{L}_\mathrm{pred}$ is the standard single-step MSE loss.

\subsection{Efficient JVP-Based Implementation}
\label{sec:method:impl}

\paragraph{Matrix-free computation.}
All four JVPs required for $\mathcal{L}_\mathrm{comm}$ and
the two JVPs plus one VJP required for $\mathcal{L}_\mathrm{norm}$
are computed using \texttt{torch.func.jvp} and \texttt{torch.func.vjp}
(PyTorch 2.0+) \citep{paszke2019pytorch}.
No explicit $n \times n$ Jacobian matrix is stored at any point.
The sequence of operations for $\mathcal{L}_\mathrm{comm}$ is:
\begin{align}
  \mathbf{u}  &= \mathbf{J}_{t+1}\,\mathbf{v},   &
  \mathbf{p}  &= \mathbf{J}_t\,\mathbf{u}
                = \mathbf{J}_t\mathbf{J}_{t+1}\,\mathbf{v},
  \label{eq:jvp-order1}
  \\
  \mathbf{w}  &= \mathbf{J}_t\,\mathbf{v},        &
  \mathbf{q}  &= \mathbf{J}_{t+1}\,\mathbf{w}
                = \mathbf{J}_{t+1}\mathbf{J}_t\,\mathbf{v},
  \label{eq:jvp-order2}
\end{align}
after which $\mathcal{L}_\mathrm{comm} = \|\mathbf{p} - \mathbf{q}\|^2$.
Each \texttt{jvp} call is a single forward pass over the latent-advance
sub-network, applied to the flattened latent.
For $\mathcal{L}_\mathrm{norm}$:
\begin{equation}
  \mathbf{J}_t\,\mathbf{v}
  \;\xrightarrow{\;\texttt{vjp}\;}
  \mathbf{J}_t^\top\!\left(\mathbf{J}_t\,\mathbf{v}\right),
  \qquad
  \mathbf{J}_t^\top\,\mathbf{v}
  \;\xrightarrow{\;\texttt{jvp}\;}
  \mathbf{J}_t\!\left(\mathbf{J}_t^\top\,\mathbf{v}\right),
  \label{eq:jvps-norm}
\end{equation}
yielding $(\mathbf{J}_t^\top\mathbf{J}_t - \mathbf{J}_t\mathbf{J}_t^\top)\mathbf{v}$
as required by~\eqref{eq:L_norm}.

\section{Related Work}
\label{sec:related}

\textbf{Long-horizon error in autoregressive neural simulators.}
The dominant framing of rollout drift is \emph{distributional shift}: a model trained on ground-truth states must, at inference, consume its own noisy predictions.
PDE-Refiner~\citep{lippe2023pderefiner} addresses this by training a denoising network that iteratively refines each rollout step, and Spectral-Refiner~\citep{cao2025spectralrefiner} fine-tunes an FNO to correct mis-resolved Fourier modes;
conditional diffusion approaches~\citep{kohl2024diffusion} go further, running many denoising iterations per step in exchange for trajectory-level error correction.
These methods target a symptom of the problem rather than its structural mechanism and impose substantial inference-time cost.
The closest training-time precursor is~\citet{mccabe2023stability}, which penalises the per-step spectral norm of the operator's Jacobian to enforce Lipschitz contraction, focusing on stability rather than transient amplification which is the goal of this paper.

\textbf{Linear-operator views and non-normal dynamics.}
Koopman approaches~\citep{mezic2005,williams2015edmd,azencot2020koopman} model the latent dynamics with a single \emph{linear} operator $K$, for which all Jacobians along a trajectory commute trivially. 
Commutativity regularisation can be read as a relaxation of the Koopman view: it asks only that the time-varying latent Jacobians \emph{share} an orthogonal eigenbasis with state-dependent eigenvalues, recovering the favourable spectral-radius propagator bound while permitting genuinely nonlinear dynamics.
The phenomenon that non-normal operators can amplify perturbations even when individually contractive originates in hydrodynamic stability theory and pseudospectra~\citep{trefethen1993,schmid2007,trefethen2005spectra,kreiss1962};
we transfer this lens from physical operators to the learned Jacobians of neural simulators and turn it into a training signal.

\textbf{Multi-step and unrolled training.}
A natural alternative to single-step training is to optimise the rollout itself \citep{list_differentiability_2025}.
The pushforward trick~\citep{brandstetter2022mpnnpde} backpropagates through one rollout step while detaching the gradient at an earlier prediction.
Such schemes mitigate covariate shift and bring training closer to the deployment regime, but their gradient signal degrades rapidly with the unroll length: \citet{scieur2023unrolling} show that differentiating through long unrolls suffers from a "curse of unrolling" in which the gradient becomes either vanishing or ill-conditioned, limiting the practical horizon over which they can shape the operator.
Commutativity regularisation acts on the same composition of Jacobians but estimates their structural defect locally, with two-step probes, side-stepping long-unroll backpropagation entirely.
The two ideas are orthogonal and can be combined: a short pushforward window controls covariate shift while comm-reg controls the per-step structural amplification.

\textbf{Training-time noise injection.}
An alternative stabilisation line injects perturbations into rollout
inputs during training, so that the model learns to recover from its
own errors. Originally introduced for graph-network particle
simulators~\citep{sanchezgonzalez2020gns} and mesh-based fluid
simulators~\citep{pfaff2021meshgraphnets}, the technique has since
been applied to turbulence emulators~\citep{stachenfeld2022coarse} to
stabilise rollouts well beyond the training horizon.
Like multi-step training, these schemes target the input distribution
rather than the structural Jacobian properties that drive transient
amplification, and are orthogonal to commutativity regularisation.

\section{Experiments}
\label{sec:experiments}

We evaluate commutativity regularisation across four spatio-temporal forecasting settings spanning 1D and 2D, integrable and chaotic, and synthetic and real data: the 1D KdV equation, the 2D barotropic vorticity equation, fine-tuning of the pretrained FourCastNet weather model on ERA5, and an OISST sea surface temperature benchmark.
In each setting we compare a baseline trained with the standard prediction loss against an otherwise-identical model trained with the same objective plus the commutator and normality penalties of Eq.~\eqref{eq:total-loss}; the rollout procedure is unchanged at inference, and we evaluate on horizons of hundreds to thousands of steps, well beyond the training window.
Across architectures (U-Net~\citep{ronneberger2015unet}, FNO~\citep{li2021fno}, U-FNO~\citep{wen2022u}, and FourCastNet's AFNO backbone~\citep{pathak2022fourcastnet}) the regularisation consistently extends the horizon over which the learned operator stays close to ground-truth trajectories, with the largest gains where transient amplification is most pronounced: deep rollouts and out-of-distribution initial conditions.

\subsection{Korteweg--de Vries equation}
\label{sec:exp:kdv}

The KdV equation is a clean probe for our central claim. It is integrable,
so any long-horizon divergence between two rollouts on the same initial
condition is attributable to the learned operator rather than to chaotic
loss of predictability. We train every model on a short window of
\emph{single}-soliton dynamics ($200$ rollout steps, $10$\,s of physical
time) and then evaluate the autoregressive rollout out to $25\times$ that
horizon, both \emph{in-distribution} and on an \emph{out-of-distribution}
test set whose initial conditions contain up to three interacting
solitons, a regime never seen during training.

We consider four backbones:
a 1D UNet, FNO, U-FNO,
PDE-Refiner~\citep{lippe2023pderefiner} as a representative
inference-time correction baseline. 
Every regime is
trained with a one-step MSE loss on the same $256$-trajectory single-soliton dataset. 
Results are shown in Figure \ref{fig:kdv_rollout_nmse} and appendix 
Tables~\ref{tab:kdv_app_unet_id}, \ref{tab:kdv_app_unet_ood}, and  \ref{tab:kdv_app_fno_full}.

\paragraph{In-distribution: commutativity reg.\ wins at long horizons.}
At very short leads, the unregularised UNet baseline is
\emph{slightly more accurate} than UNet+CR
: the regulariser pays a small
single-step cost. 
This ordering reverses already inside the training
window and the gap then widens monotonically. 
By step 1000 UNet+CR shows a $5\times$ improvement over the baseline. 
PDE-Refiner is the strongest
unregularised baseline at the shorter horizons, it is more accurate
than UNet+CR up to $\sim$$1000$ steps, paying $5\times$ inference-time
cost, but it eventually accumulates error and at
step $5000$, it is $5\times$ worse than the training-time regulariser despite
its much inference cost.

\paragraph{Out-of-distribution: the gap grows on multi-soliton ICs.}
On the multi-soliton OOD split the unregularised UNet and PDE-Refiner
both experience severe degradation, failing to maintain rollout coherence.
UNet+CR is essentially unaffected by the distribution shift.
The regularised model is the only one that stays
on a coherent multi-soliton trajectory at the longest horizons. 
The
qualitative space-time visualization is in
Figures \ref{fig:kdv-indist} and \ref{fig:kdv-ood}.
\begin{figure}[t]
  \centering
  \includegraphics[width=0.33\linewidth]{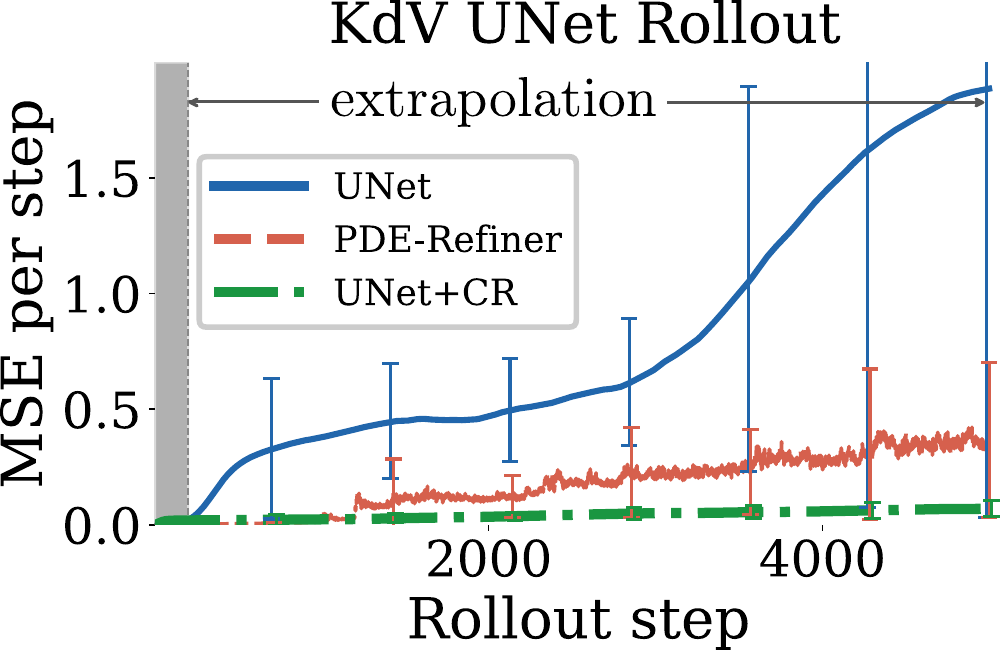}
  \includegraphics[width=0.31\linewidth]{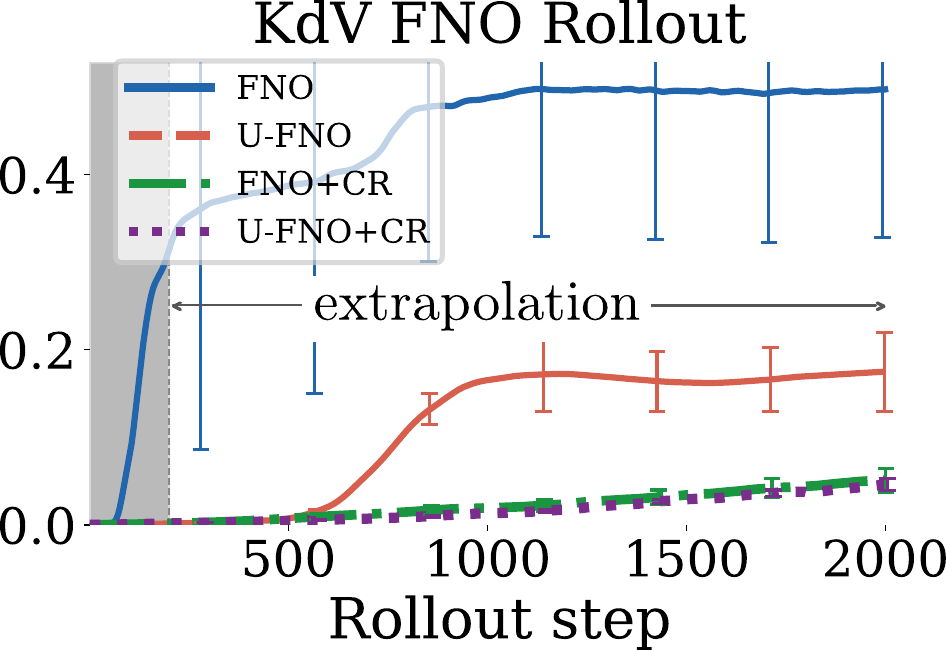}
  \includegraphics[width=0.34\linewidth]{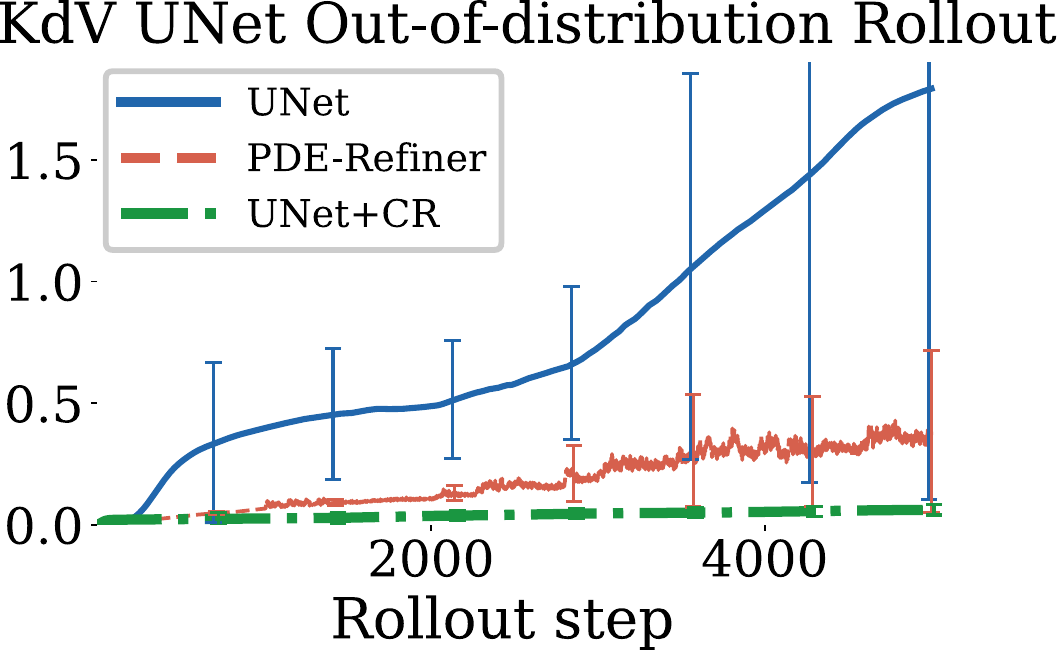}
  \caption{KdV rollout nMSE vs.\ time, averaged over 50 held-out test
           trajectories. The dashed line marks the training horizon
           ($t = 10$\,s, $200$ steps); everything to its right is
           extrapolation in rollout length.
           onward.
           (Cf. Tables~\ref{tab:kdv_app_unet_id}, \ref{tab:kdv_app_unet_ood}, \ref{tab:kdv_app_fno_full})}
  \label{fig:kdv_rollout_nmse}
\end{figure}

\paragraph{Spectral backbones: comm.\ reg.\ recovers an unstable rollout.}
On the same data, the FNO and U-FNO backbones exhibit
 a stronger version of the same effect
  (Figure~\ref{fig:kdv_rollout_nmse}, middle; detailed 
results in Table~\ref{tab:kdv_app_fno_full}).

\begin{wrapfigure}{r}{0.55\textwidth}
\vspace{-1.5em}
  \includegraphics[width=0.55\textwidth]{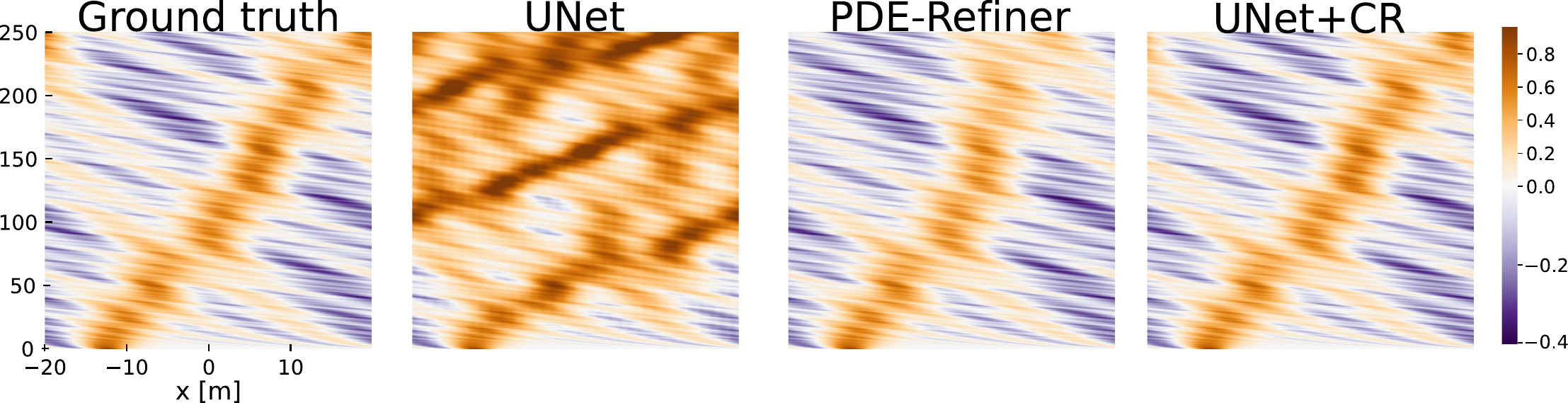}
  \vspace{-1.7em}
  \caption{KdV UNet variant rollout }
  \label{fig:kdv-indist}
  \vspace{0.5em}
  \includegraphics[width=0.55\textwidth]{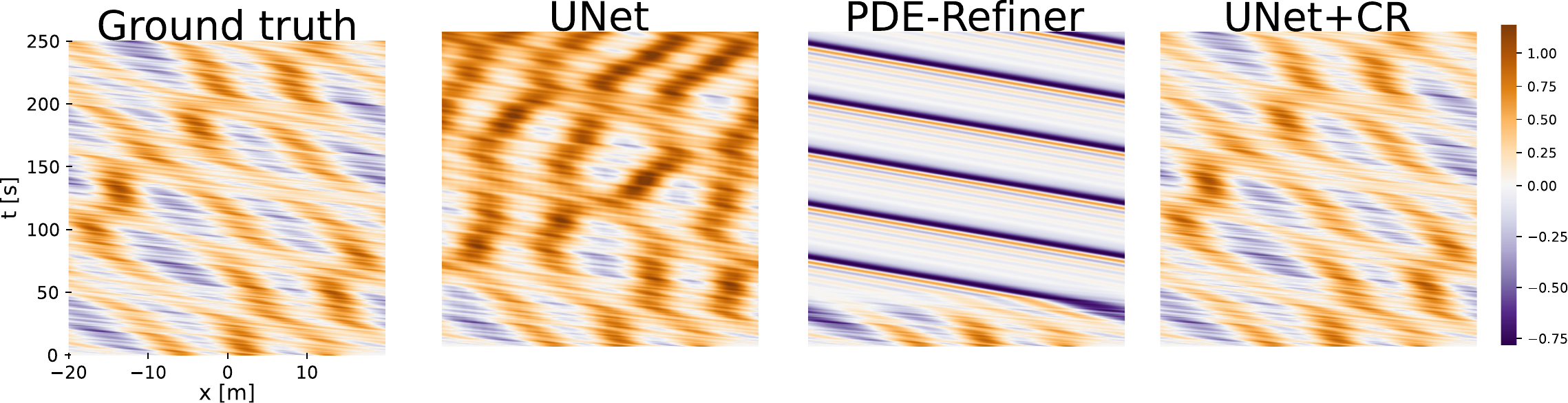}
  \vspace{-1.5em}
  \caption{KdV UNet variant OOD rollout}
  \label{fig:kdv-ood}
  \vspace{0.5em}
  \includegraphics[width=0.55\textwidth]{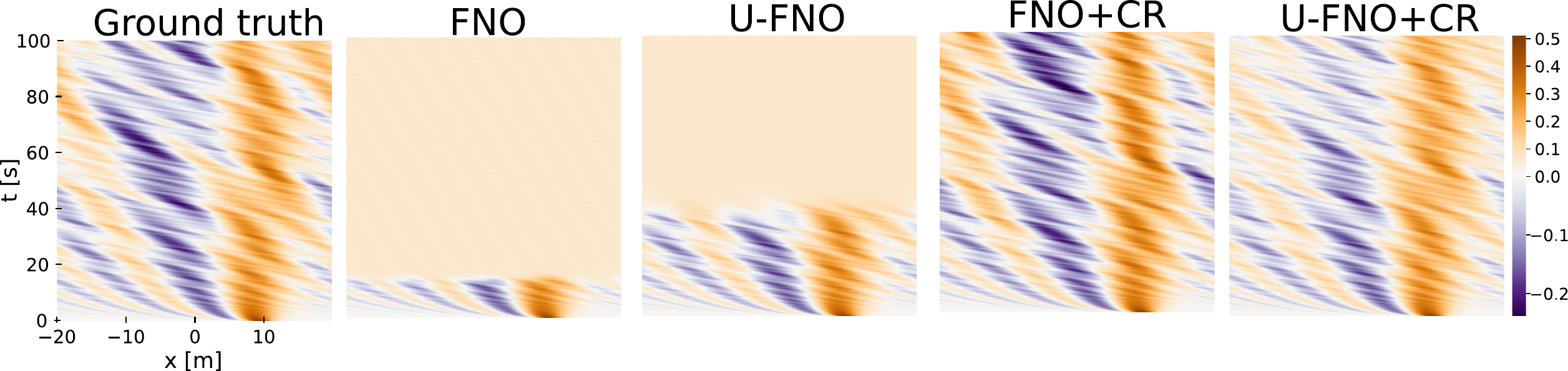}
  \vspace{-1.5em}
  \caption{KdV FNO variant Rollouts }
  \label{fig:kdv_rollout_fno}
  \vspace{-1em}
\end{wrapfigure}
The unregularised FNO has
an excellent one-step MSE 
on the validation set, $\sim$$4\times$ better per step than its regularised
counterpart) yet its autoregressive rollout becomes unstable around
step $100$ 
and loses coherence already by the end of the
training window. 
The U-FNO baseline holds longer, 
until step $500$, before collapsing. 
Both commutativity-regularised variants stay bounded throughout the
$2000$-step evaluation window,
comparable to the regularised UNet at the same horizon and is reached on a backbone that diverges entirely without the
penalty. 
On the spectral architectures the regulariser thus does more
than improve a stable baseline: it converts a divergent autoregressive
rollout into a bounded one, exactly the qualitative behaviour the
propagator analysis of Section~\ref{sec:background:rollout} predicts
when non-normal Jacobian products are left unconstrained.

\paragraph{Summary.}
Across all four backbones, commutativity regularisation trades a small
amount of one-step accuracy for a long-horizon rollout that is bounded
and, beyond the training window, between $5\times$ and several orders
of magnitude more accurate than its unregularised counterpart. It
matches PDE-Refiner at intermediate horizons without the latter's
$5\times$ inference-time overhead, and outperforms it past $\sim$$2000$
steps. The OOD experiment shows that the gain is robust to a regime
shift in the initial conditions, suggesting that the regulariser shapes
a structural property of the learned operator rather than memorising a
training distribution. Per-trajectory space--time snapshots, the full
per-horizon nMSE numbers, and the architectural and training details
are in Appendix~\ref{app:exp:kdv}.

\subsection{Barotropic Vorticity Equation}
\label{sec:exp:bve}
We evaluate on a chaotic 2D fluid PDE: the
barotropic vorticity equation (BVE). The system is dissipative and chaotic: small perturbations
amplify exponentially in time,  so any improvement  must model intrinsic divergence. 

\begin{wrapfigure}{r}{0.4\textwidth}
  \includegraphics[width=0.4\textwidth]{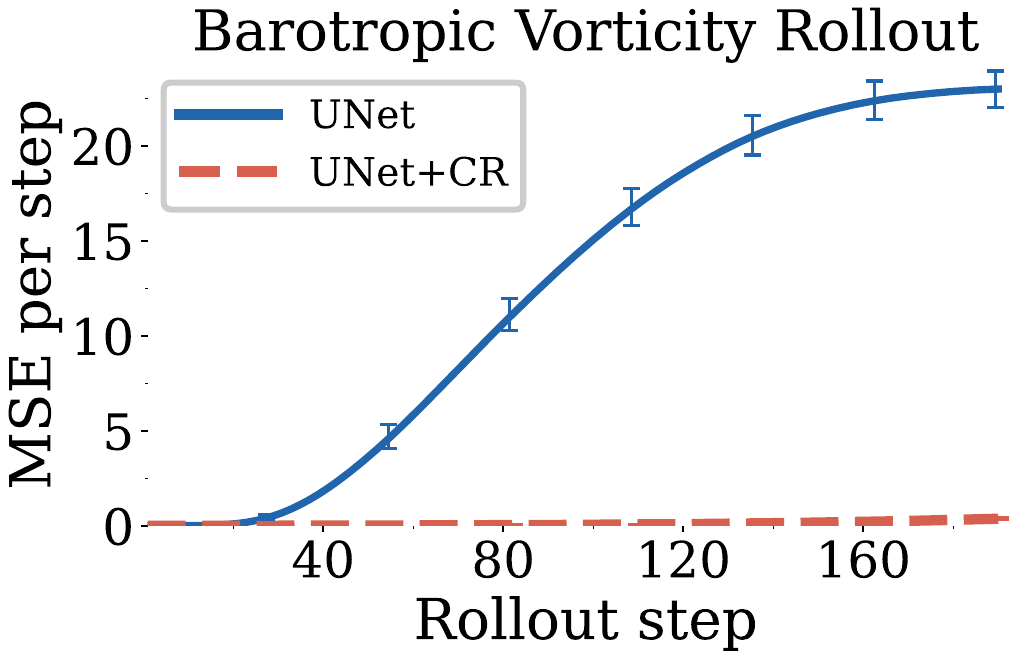}
  \caption{BVE rollout RMSE on the held-out test set.
           }
  \label{fig:bve_rollout_rmse}
\end{wrapfigure}

\paragraph{Setup.}
The backbone is a 2D UNet with circular padding,
and an $8{\times}8{\times}256$ bottleneck on which the regulariser acts.
Both regimes are trained with one-step MSE on $200$-frame
trajectories. 
Further details are deferred to Appendix~\ref{app:exp:bve}.

\paragraph{Result: the baseline destabilises inside the training window.}
Figure~\ref{fig:bve_rollout_rmse} and Table~\ref{tab:bve_full_horizons} report
the comparison. The two regimes are essentially tied at one step
but they quickly separate fast. 
By $t{=}1$\,s ($20$ steps) the baseline is already $5\times$ worse than the regularised model.
$\mathrm{RMSE}{=}1$, error variance equal to signal variance,
and from $t{\approx}3$\,s onward it saturates at
$\mathrm{RMSE}\!\approx\!3$--$6$, 
where the rollout has decorrelated from the ground truth.
The regularised model stays below $\mathrm{RMSE}{=}1$ throughout the
training horizon. 
At the end of the
$10$\,s window the regularised RMSE is $9\times$ smaller than the
baseline; in nMSE terms (the operationally relevant unit on this
normalised field) the gap is $82\times$.

The qualitative picture is in
the snapshot grid in Figure \ref{fig:rollout_bvort} and Appendix~\ref{app:exp:bve:snapshots}: the
baseline visibly amplifies vorticity to several times the natural
scale of the flow well before the end of the training horizon, while
the regularised rollout preserves coherent vortex structures and the
correct amplitude band throughout.

\paragraph{Summary.}
On a 2D chaotic PDE the unregularised UNet rollout becomes unstable
inside the training horizon. The same training-time regulariser that
controls long-horizon drift on KdV turns this instability into a
bounded, on-attractor rollout, with a $9\times$ RMSE improvement at
the end of the $10$\,s horizon and a $\sim$$80\times$ improvement in
nMSE. Together with the KdV results, this validates the effectiveness the regulariser across 1D and 2D and integrable vs.\ chaotic. 
 In both regimes, single-step MSE and rollout error rank the regimes in opposite orders, and a Jacobian-product-aware regulariser recovers the long-horizon
behaviour that the standard one-step training objective leaves
unconstrained. Per-step RMSE numbers, full data generation,
architecture and hyperparameters, and more visualizations are
in Appendix~\ref{app:exp:bve}.

\begin{figure}[t]%
  \centering
  \includegraphics[width=0.6\textwidth]{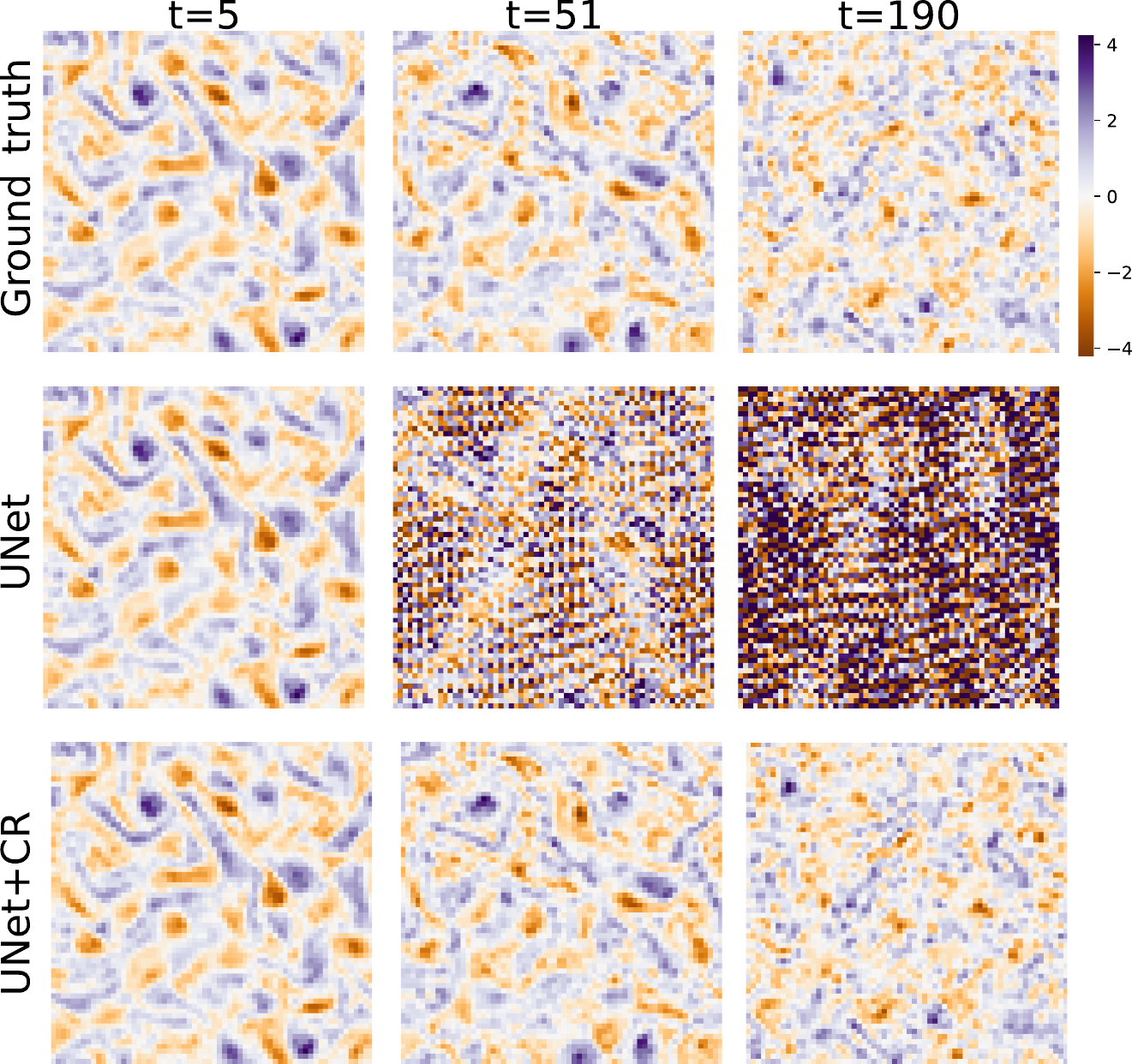}
  \caption{Barotropic vorticity unregularized (middle) and regularized (bottom) UNet rollout (Cf. Appendix \ref{app:exp:bve}).}
  \label{fig:rollout_bvort}
\end{figure}

\subsection{FourCastNet on ERA5}
\label{sec:exp:fcn}

We test commutativity regularisation on a large pretrained autoregressive
weather model: FourCastNet v1~\citep{pathak2022fourcastnet}, an AFNO
that maps the 20-channel ERA5 \citep{https://doi.org/10.1002/qj.3803, rasp2023weatherbench} state at $720{\times}1440$
($0.25^{\circ}$) to the state $6$\,h later. The question is whether
the regulariser improves long-horizon rollout when FCN is fine-tuned
on a small, in-distribution slice of its own training data, so that any gain has to come from the
structural effect of the regulariser on the learned Jacobians. 
The main finding is that plain supervised fine-tuning on this slice actively \emph{worsens} the long-horizon
near-surface temperature forecast, while fine-tuning with commutativity regularisation on the same data turns this regression
into a substantial improvement.
Results are shown in Figure \ref{fig:fcn_rmse_curves} and Tables \ref{tab:fcn_rmse_full_2018} and \ref{tab:fcn_rmse_full_2019}.

\paragraph{Setup.}
Starting from the pretrained FCNv1 \footnote{\url{https://github.com/NVlabs/FourCastNet}} backbone, we fine-tune three years (2013-2015) inside FCN's own 1979--2015 training range
 and compare against
(i) the frozen pretrained model and (ii) a plain supervised
fine-tune on 2013--2015 with the same optimiser and schedule but no regularization. 
The regularization is applied on the deep
AFNO blocks: the first $8$ blocks are run once and detached, and the
penalties are evaluated on the last two blocks. 
In this case, we apply the regularizer to layer normalized blocks.
We evaluate
40-step ($10$-day) autoregressive rollouts on ERA5~$2018$ at $41$ initial conditions, following \cite{pathak2022fourcastnet}, and report latitude-weighted RMSE in physical
units on the variables $z_{500}$ and $t_{2\mathrm{m}}$. 
A second
held-out evaluation on ERA5~$2019$ shows the same improvement for the three-year finetuning. 
Architecture, hyperparameters, visualizations and other details are
in Appendix~\ref{app:exp:fcn}.

\paragraph{Plain finetuning worsens
                  $t_{2\mathrm{m}}$, the regulariser recovers it.}
Table~\ref{tab:fcn_rmse_full_2018} reports the $t_{2\mathrm{m}}$ rollout
RMSE on ERA5~$2018$. Through day $5$ the three regimes are tied
within $\sim$$5\%$. From day $7$ onwards they split sharply: the
plain finetune \emph{regresses} relative to the frozen model
, while the regularised finetune on the same data improves on the frozen model.
The $2019$ replication (Appendix~\ref{app:exp:fcn:numbers}) shows
the same pattern. 
On the upper-air channels $z_{500}$, $t_{850}$, $u_{850}$ (all four
in Appendix~\ref{app:exp:fcn:numbers}) the regularised and frozen
models are within $\sim$$1$--$3\%$ at every lead, while the plain
finetune again drifts at the longest lead (e.g.\ $t_{850}$
day $10$: $4.26$\,K frozen, $6.33$\,K plain finetune, $4.14$\,K
regularised). The regulariser thus matches the frozen model on
the variables FCN already gets right and substantially improves
on the variable it most struggles with.

\begin{figure}[t]
  \vspace{-0.6em}
  \includegraphics[width=\linewidth]{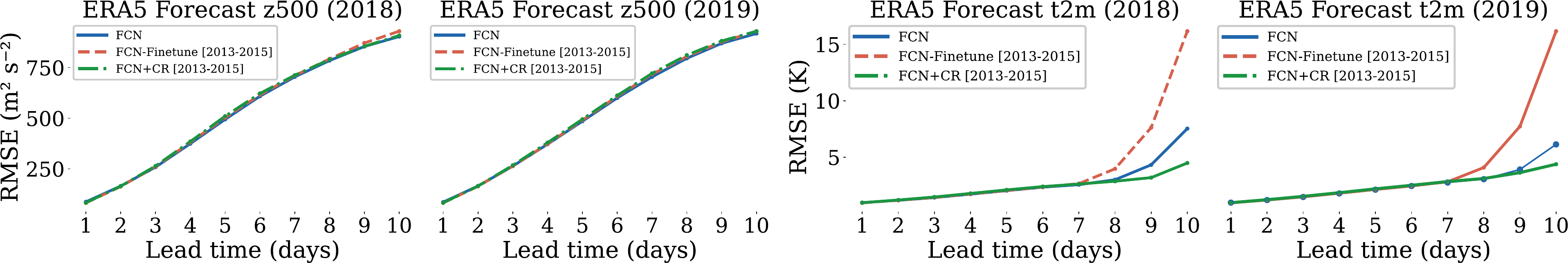}
  \caption{ERA5 rollout RMSE vs.\ lead time on $t_\mathrm{2m}$ and $z_{500}$ for held-out years 2018 and 2019 together with finetuning without regularization.
           Finetuning-only destroys long-horizon $t_\mathrm{2m}$;
           the same data with commutativity regularisation pulls
           well below the frozen FCN baseline.}
  \label{fig:fcn_rmse_curves}
  \vspace{-0.6em}
\end{figure}

\paragraph{Summary.}
On a $74$\,M-parameter pretrained weather model, fine-tuning on a
small in-distribution slice of its own training years with the same
training-time regulariser used for the synthetic PDEs flips a
catastrophic $3{\times}$ degradation of $t_{2\mathrm{m}}$ at day~$10$
into a $41\%$ improvement (and a comparable improvement on the
held-out 2019 year), without exposure to new data and without
inference-time cost. Together with the KdV and BVE results, this
brings the validation from $1$D integrable and
$2$D chaotic synthetic settings up to a high-resolution real-world
weather model, on the variable for which the iterated rollout most
visibly fails. 
The full per-channel, per-lead, multi-year tables,
the regulariser implementation details and the architectural
hyperparameters are in Appendix~\ref{app:exp:fcn}.

\subsection{Sea Surface Temperature Forecast}
\label{sec:exp:sst}
We train and compared a regularized model from scratch on a real, observation-derived field:
weekly NOAA OISST sea-surface temperature on a global $1^\circ$
grid~\citep{reynolds2002oisst}. A 2D UNet
($k_{\text{in}}{=}1$, depth $4$, base width $64$) is trained with
one-step MSE on the $1200$-week training split, with and without the
latent commutativity / normality penalties of
Section~\ref{sec:method} on its bottleneck. 
We initialise from the
first frame of the $377$-week test split and roll autoregressively
for $376$ weeks ($\approx 7.2$\,years). Setup, architecture and
hyperparameters are in Appendix~\ref{app:exp:sst}.

The two models are tied at one-week lead, 
i.e.\ the regulariser does not hurt short-lead accuracy. They diverge
from $\sim 3$ months onward: the baseline drifts in phase relative
to the annual cycle and amplifies anomalies, while the regularised
rollout tracks the seasonal cycle out to the end of the $\sim 7$-year
window. Averaged over the full $376$-step rollout the regularised
model is $3.0\times$ more accurate. Together with the BVE result,
this shows that the same regulariser controls two qualitatively
different long-horizon failure modes, chaotic energy saturation
on BVE, and slow phase drift on a forced near-periodic real-world
field. 

\begin{figure}[t]%
\centering
  \includegraphics[width=0.4\textwidth]{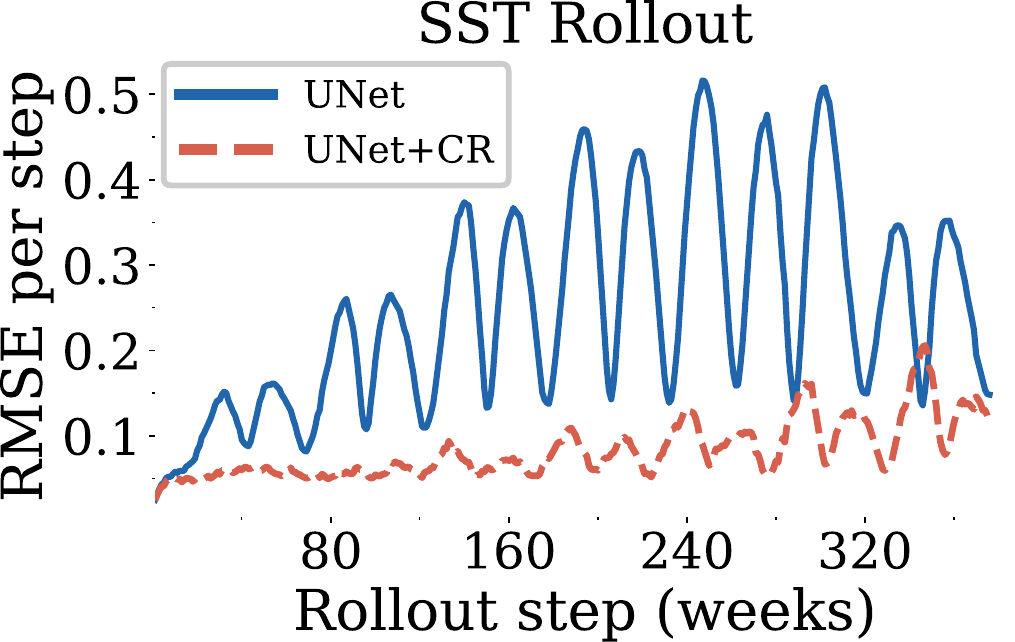}
  \caption{SST rollout RMSE (normalised units) versus lead time (Cf. Appendix \ref{app:exp:sst}).}
  \label{fig:sst:rollout-rmse}
\end{figure}

\section{Conclusion}
\label{sec:conclusion}
We identified \emph{transient amplification} of perturbations along rollout trajectories, driven by non-normal and non-commuting latent Jacobians, as a structural mechanism for long-horizon error growth in autoregressive neural simulators and responsible from distributional shift.
On the basis of this analysis we introduced \emph{commutativity regularisation}: a pair of JVP-based penalties on per-step Jacobian normality and across-step commutativity, supported by a propagator bound that replaces the spectral-norm $\alpha^T$ rate by the spectral-radius $\rho^T$ rate as the two penalties go to zero.
Across KdV, the barotropic vorticity equation, sea surface temperature data, and fine-tuning of the pretrained FourCastNet weather model on ERA5, the method consistently improves long-horizon accuracy at zero inference-time cost, and pushes regularised models well past the rollout horizons at which baselines diverge.

\paragraph{Limitations and outlook.}
The propagator bound is linearisation-based and assumes that the local Jacobians are representative of the rollout error map; far from the training distribution, where the linearisation itself is suspect, our guarantees become only suggestive.
The penalties are estimated stochastically with single-probe JVPs, and per-step normality together with adjacent-step commutativity is a tractable surrogate for the all-pairs joint diagonalisability condition that the theory asks for.
The method also targets the \emph{structural} component of error growth and is therefore complementary to, rather than a replacement for, refinement and diffusion methods that address distributional shift; we expect the strongest results from combining the two.
Finally, the regulariser is most natural for architectures with an explicit latent bottleneck through which temporal dynamics must pass; choosing the regularised subspace for purely spectral or attention-based simulators, and scaling the method to global weather and climate models at higher resolution, are the directions we view as most pressing.

\newpage
\section*{Acknowledgements}

\begin{wrapfigure}[3]{r}{0.2\linewidth}
  \vspace{-0.2in}
  \hspace{0.15in}
    \includegraphics[width=0.65\linewidth]{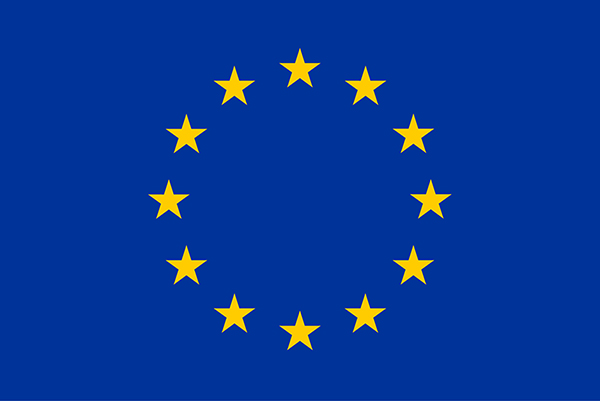}
\end{wrapfigure}

This project has received funding from the European Union’s Horizon 
2020 research and innovation programme under the Marie Skłodowska-Curie Grant Agreement No. 101034413.
\\\\

\newpage

\bibliography{references,references_extra}
\newpage
\appendix

\section{Proof of theorem \ref{thm:propagator}}
\label{app:method:prop-proof}

We first record the exact result when the Jacobians commute and are individually normal.

\begin{proposition}[Exact commuting case]
  \label{prop:exact}
  Let $J_0,\ldots,J_{T-1} \in \mathbb{R}^{n\times n}$ be simultaneously diagonalisable
  as $J_t = U \Lambda_t U^\top$ for a common orthogonal matrix $U$ and diagonal
  $\Lambda_t = \operatorname{diag}(\lambda_{1,t},\ldots,\lambda_{n,t})$.
  If $\rho = \max_{i,t}|\lambda_{i,t}|$, then
  \begin{equation}
    \|\Phi_T\|_2 \;\le\; \rho^T.
    \label{eq:exact-commuting}
  \end{equation}
\end{proposition}

\begin{proof}
  $\Phi_T = U\,\Lambda_{T-1}\cdots\Lambda_0\,U^\top$.
  Since $U$ is orthogonal, $\|\Phi_T\|_2 = \|\Lambda_{T-1}\cdots\Lambda_0\|_2$.
  The product of diagonal matrices is diagonal with $(i,i)$ entry
  $\prod_{t=0}^{T-1}\lambda_{i,t}$, so
  $\|\Lambda_{T-1}\cdots\Lambda_0\|_2 = \max_i |\prod_t \lambda_{i,t}|
  \le \rho^T$.
\end{proof}

\begin{remark}
The improvement over the crude submultiplicativity bound $\alpha^T$
  (where $\alpha = \max_t\|J_t\|_2 \ge \rho(J_t) \ge \rho$) is real whenever the
  Jacobians are \emph{non-normal}: a non-normal matrix satisfies $\rho(J) < \|J\|_2$,
  so $\rho < \alpha$ and $\rho^T \ll \alpha^T$ for large $T$.
  Proposition~\ref{prop:exact} therefore shows that commutativity (plus normality) lets
  the propagator grow at the \emph{spectral-radius} rate rather than the
  \emph{spectral-norm} rate.
\end{remark}

\proptheorem*
\begin{proof}%
When $\varepsilon = \eta = 0$, each $J_t$ is normal and all pairs commute.
Commuting normal matrices are simultaneously diagonalizable by a common
orthogonal matrix~\citep{horn1985}, giving $\|\Phi_T\|_2 \le \rho^T$
by Proposition~\ref{prop:exact}.

For $\varepsilon, \eta > 0$, the joint conditions (iii)--(iv) imply that $\{J_t\}$
lies within distance $\delta(\varepsilon,\eta)$ of the closed set of simultaneously
orthogonally diagonalizable families, i.e.\ there exist
$J'_t = U\Lambda_t U^\top$ (orthogonal $U$, diagonal $\Lambda_t$ with
$|\lambda_{i,t}| \le \rho$) satisfying $\|J_t - J'_t\|_2 \le \delta(\varepsilon,\eta)$.

Write $J_t = J'_t + E_t$ with $\|E_t\|_2 \le \delta \equiv \delta(\varepsilon,\eta)$.
Expanding the product to first order in $\{E_t\}$:
\[
  \Phi_T
  \;=\; \prod_{t=0}^{T-1}(J'_t + E_t)
  \;=\; \underbrace{\prod_{t=0}^{T-1} J'_t}_{=:\,\Phi'_T}
    \;+\; \sum_{k=0}^{T-1}
      \underbrace{J'_{T-1}\cdots J'_{k+1}}_{L_k} E_k
      \underbrace{J_{k-1}\cdots J_0}_{R_k}
    \;+\; O(\delta^2).
\]
Since $J'_t = U\Lambda_t U^\top$, we have $\|\Phi'_T\|_2 = \|\Lambda_{T-1}\cdots\Lambda_0\|_2
\le (\rho + \delta)^T \le \rho^T + T\rho^{T-1}\delta \le \rho^T + T\alpha^{T-1}\delta$.
Each cross-term satisfies $\|L_k E_k R_k\|_2 \le \alpha^{T-1}\delta$.
Summing the $T$ cross-terms and absorbing the $O(\delta^2)$ remainder into the leading
terms for small $\delta$:
\[
  \|\Phi_T\|_2
  \;\le\; \rho^T + T\alpha^{T-1}\delta + T\alpha^{T-1}\delta
  \;=\; \rho^T + 2T\alpha^{T-1}\delta(\varepsilon,\eta). \qedhere
\]
\end{proof}

\begin{remark}[Why both conditions are needed]
  Condition~(iii) alone (commutativity, without normality) yields commuting
  approximants $J'_t$ that share an eigenbasis, but that basis need not be
  orthogonal, so the $\|\Phi'_t\|_2$ bound would require a conditioning factor for the eigenbasis.
  Condition~(iv) alone (normality, without commutativity) ensures each $J_t$
  individually has $\|J_t^k\|_2 = \rho(J_t)^k$, but products of distinct
  normal matrices need not satisfy $\|\Phi_T\|_2 \le \rho^T$.
  The two conditions are complementary: together they drive the sequence
  toward a simultaneously orthogonally diagonalizable family.
\end{remark}

\begin{remark}[Spectral norm alone is insufficient]
  Condition~(i) without~(iii)--(iv) does not prevent transient amplification.
  There exist sequences $\{J_t\}$ with $\|J_t\|_2 = \alpha < 1$ yet
  $\|\Phi_T\|_2 \gg \alpha^T$ for $T$ in the transient regime, due to
  non-commuting, non-normal interactions.
\end{remark}

\section{Regularizer configuration}
\label{app:reg-config}

For very large models where the regularizing the latent advance in infeasible (such as FourCastNet), we instead regularize a chosen block in the network.
A schematic of this is show in Figure \ref{fig:reg-arch-block}.

\begin{figure}[h!]
  \centering
  \includegraphics[width=0.8\textwidth]{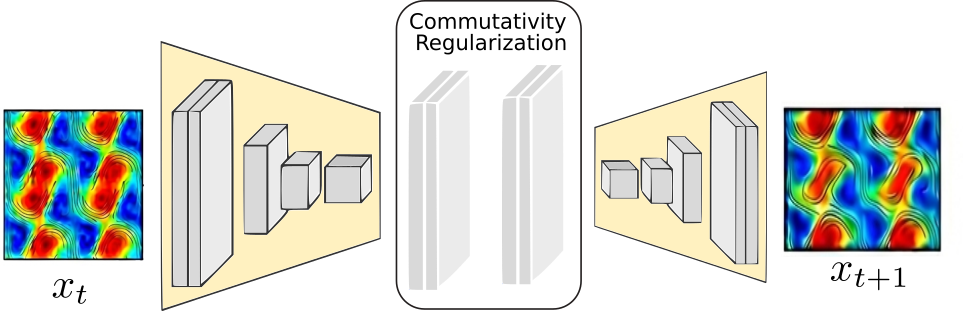}
  \caption{Latent block commutativity regularization. Use only for the FourCastNet experiment in this paper.}
  \label{fig:reg-arch-block}
\end{figure}

\section{Reducing computational overhead}
\label{app:overhead}
Three practical approximations reduce cost with negligible quality loss.
(a) \emph{Temporal subsampling}: the regularization terms are evaluated only every
$K \in \{1,5,10\}$ minibatches depending on the experiment, since Jacobian structure changes slowly relative to
individual gradient steps.
(b) Jacobian vector products are chosen on a smaller sample (about half) of each minibatch, and
(c) \emph{Single shared probe}: one probe vector $\mathbf{v}$ is drawn per minibatch
and reused for all examples, rather than drawing per-example vectors.
With these approximations the additional training wall-clock cost is
approximately $20$--$30\%$ over the baseline.
The inference cost is \emph{exactly zero}: the regularizer affects only the
learned weights $\theta$, not the rollout procedure.

\section{KdV experiment: details}
\label{app:exp:kdv}

This appendix expands the KdV experiment of
Section~\ref{sec:exp:kdv}: data and OOD construction, the four
backbone architectures, training and regulariser hyperparameters, and the full per-horizon nMSE numbers
underlying Figure~\ref{fig:kdv_rollout_nmse} and
Tables~\ref{tab:kdv_app_unet_id}, \ref{tab:kdv_app_unet_ood} and \ref{tab:kdv_app_fno_full}.

\subsection{Data generation and OOD construction}
\label{app:exp:kdv:data}

\paragraph{Equation and solver.}
We solve the 1D Korteweg--de Vries equation
\begin{equation}
  u_t + u\,u_x + u_{xxx} \;=\; 0,
  \qquad x \in [-20, 20],
  \qquad u(\cdot,t)\ \text{periodic},
  \label{eq:kdv}
\end{equation}
on a uniform grid of $N=256$ points. Spatial derivatives are taken
pseudospectrally,
$u_x=\mathcal F^{-1}[ik\hat u]$ and
$u_{xxx}=\mathcal F^{-1}[(ik)^3\hat u]$, and the resulting ODE system is
integrated with SciPy's adaptive RK45 (\texttt{solve\_ivp},
$\texttt{rtol}=10^{-6}$). Snapshots are stored on a uniform time grid
of $\Delta t = 0.05$\,s.

\paragraph{Single-soliton training set (in-distribution).}
Initial conditions are random single solitons,
\begin{equation}
  u_0(x) \;=\; A\,\mathrm{sech}^{2}\!\bigl((x-x_0)/w\bigr),
  \qquad
  A \sim \mathcal U[0.5, 2.0],\quad
  w \sim \mathcal U[0.5, 2.0],\quad
  x_0 \sim \mathcal U[-15, 15].
\end{equation}
We generate $256$ training and $120$ validation trajectories of $200$
steps each ($10$\,s of physical time) for fitting all four backbones,
and an additional $50$ trajectories of $5000$ steps for the long-horizon
in-distribution evaluation.

\paragraph{Multi-soliton OOD set.}
The OOD test set is generated by the same solver but with initial
conditions consisting of a uniformly random number of solitons in
$\{1,2,3\}$, each with independently sampled amplitude, width and
centre from the training distribution. The $50$ OOD trajectories are
$5000$ steps long. Models trained only on single solitons are thus
evaluated on collisions and overtakings they have never seen during
training.

\subsection{Architectures}
\label{app:exp:kdv:arch}

All four backbones consume a single channel
($u(x,t_k)$) and produce a single channel ($u(x,t_{k+1})$) on the same
$256$-point periodic grid.

\paragraph{1D UNet.}
A circular-padded 1D UNet,
$\sim$$1.39$\,M parameters): base channel width $32$, channel
multipliers $(1,2,4,8)$, depth $4$, $3{\times}1$ kernels, GELU
activations, no normalisation, residual two-conv blocks at every
encoder/decoder level, strided $2{\times}1$ convolutions for
down-sampling and transposed convolutions for up-sampling. Skip
connections at every encoder level. The bottleneck has shape
$(256, 16) = 4096$ dimensions and is the latent state on which the
regulariser acts; the latent advance operator is
$F(z, \mathrm{skips}) =
\mathrm{encode}\bigl(\mathrm{decode}(z, \mathrm{skips})\bigr)$, so the
regularised Jacobian naturally couples the bottleneck channels with
the skip-feature channels.

\paragraph{Fourier Neural Operator (FNO).}
A four-block FNO, with $64$ retained Fourier modes per layer,
channel width $128$, GELU activations, and a final two-layer pointwise
projection. The first two blocks define the encoder
($z=\mathrm{encode}(u)$ of shape $(64,256)=16384$ dims, no skip
connections); the last two define the decoder. The latent advance
operator differentiated by the regulariser is
$F(z) = \mathrm{encode}(\mathrm{decode}(z))$.

\paragraph{U-FNO.}
The U-FNO backbone replaces the last
two FNO blocks with U-FNO blocks that sum a spectral convolution, a
small two-level $1$D UNet on the same feature map, and a $1{\times}1$
pointwise convolution before activation. Channel width $128$, $64$
retained Fourier modes, GELU activations. The regulariser acts on the
encoder/decoder split in the same way as for the plain FNO.

\paragraph{PDE-Refiner.}
We use the PDE-Refiner formulation~\citep{lippe2023pderefiner}
with a single denoising network conditioned on the previous state, the
current (noisy) prediction, and a step index $k\in\{0,\ldots,M\}$. 
The backbone is identical to
the UNet above; we use $M=4$ refinement iterations per rollout step
and the geometric noise schedule of \cite{lippe2023pderefiner} with
$\sigma_{\min}=10^{-7}$. PDE-Refiner therefore costs $M{+}1=5\times$
backbone evaluations per rollout step at inference; commutativity
regularisation does not modify the rollout procedure and adds no
inference-time cost.

\subsection{Training and regulariser hyperparameters}
\label{app:exp:kdv:hp}

\begin{table}[h]
  \centering
  \caption{KdV training hyperparameters, identical between baseline and
           regularised regimes except for the loss term. PDE-Refiner
           uses the same optimiser settings as the UNet and adds the
           denoising loss schedule of~\cite{lippe2023pderefiner}.}
  \label{tab:kdv_hp}
  \small
  \begin{tabular}{ll}
    \toprule
    Optimiser            & AdamW \\
    Peak learning rate   & $3 \times 10^{-4}$ \\
    Weight decay         & $10^{-5}$ \\
    Schedule             & cosine annealing to $10^{-7}$ \\
    Epochs               & $500$ (UNet, FNO, U-FNO) / $500$ (PDE-Refiner) \\
    Batch size           & $256$ \\
    Training trajectory length & $200$ steps ($10$\,s) \\
    Loss (one-step)      & $\|F_\theta(u_t) - u_{t+1}\|^2$ (MSE) \\
    \midrule
    Regulariser          & latent-space comm.\ + normality \\
    $\lambda_\mathrm{c}$ & $10^{-4}$ \\
    $\lambda_\mathrm{n}$ & $10^{-4}$ \\
    JVP frequency        & every $10$th minibatch \\
    Probe vector $v$     & i.i.d.\ Gaussian, fresh per evaluation \\
    Adjacent pair        & $(z_t, z_{t+1}{=}G_\theta(z_t))$ \\
    \bottomrule
  \end{tabular}
\end{table}

\paragraph{Regulariser construction (same on every backbone).}
At a regulariser-active minibatch we draw a fresh sample $u_t$,
encode it to obtain $z_t$, then decode-then-encode the pair
$(z_t,\mathrm{skips}_t)$ to obtain the model's own one-step image
$z_{t+1}=G_\theta(z_t)$ (and the corresponding skip activations for
the UNet). The commutator and normality penalties are computed as
\[
  \mathcal{L}_{\mathrm{comm}}
   \;=\; \|[J_t, J_{t+1}]\,v\|^2,
  \qquad
  \mathcal{L}_{\mathrm{norm}}
   \;=\; \|(J_t^\top J_t - J_t J_t^\top)\,v\|^2,
\]
on the latent state space, with $J_\bullet=\partial F/\partial z$
evaluated at the appropriate latent state, each Jacobian--vector
product computed by a single \texttt{torch.func.jvp} call (and one
\texttt{vjp} for the normality term). Subsampling the regulariser to
every tenth minibatch keeps the training-time overhead within
$\sim$$10$--$20\%$ of the baseline per epoch on the same hardware.

\subsection{Full per-horizon results: UNet variants}
\label{app:exp:kdv:numbers_unet}

\begin{table}[h]
  \centering
  \caption{KdV final-step nMSE for the UNet variants on the
           single-soliton in-distribution test set
           (50 trajectories, mean over $3$ seeds).
           Training horizon is $200$ steps. Lower is better.}
  \label{tab:kdv_app_unet_id}
  \scriptsize
  \setlength{\tabcolsep}{3pt}
  \begin{tabular}{lcccccccc}
    \toprule
    Step & $50$ & $100$ & $200$ & $500$ & $1000$ & $2000$ & $3000$ & $5000$ \\
    \midrule
    UNet
      & $2.9\!\times\!10^{-4}$
      & $2.1\!\times\!10^{-3}$
      & $2.1\!\times\!10^{-2}$
      & $2.6\!\times\!10^{-1}$
      & $3.8\!\times\!10^{-1}$
      & $4.7\!\times\!10^{-1}$
      & $6.7\!\times\!10^{-1}$
      & $1.9$ \\
    PDE-Refiner
      & $2.2\!\times\!10^{-4}$
      & $6.1\!\times\!10^{-4}$
      & $1.5\!\times\!10^{-3}$
      & $4.2\!\times\!10^{-3}$
      & $1.2\!\times\!10^{-2}$
      & $1.1\!\times\!10^{-1}$
      & $2.1\!\times\!10^{-1}$
      & $3.7\!\times\!10^{-1}$ \\
    UNet + comm.\ reg.
      & $1.3\!\times\!10^{-2}$
      & $1.8\!\times\!10^{-2}$
      & $1.9\!\times\!10^{-2}$
      & $2.0\!\times\!10^{-2}$
      & $2.3\!\times\!10^{-2}$
      & $3.5\!\times\!10^{-2}$
      & $5.1\!\times\!10^{-2}$
      & $6.9\!\times\!10^{-2}$ \\
    \bottomrule
  \end{tabular}
\end{table}

\begin{table}[h]
  \centering
  \caption{KdV final-step nMSE for the UNet variants on the
           multi-soliton out-of-distribution test set
           ($50$ trajectories with up to three solitons,
           mean over $3$ seeds). Training horizon is $200$ steps.
           Lower is better.}
  \label{tab:kdv_app_unet_ood}
  \scriptsize
  \setlength{\tabcolsep}{3pt}
  \begin{tabular}{lcccccccc}
    \toprule
    Step & $50$ & $100$ & $200$ & $500$ & $1000$ & $2000$ & $3000$ & $5000$ \\
    \midrule
    UNet
      & $2.0\!\times\!10^{-4}$
      & $2.2\!\times\!10^{-3}$
      & $2.1\!\times\!10^{-2}$
      & $2.6\!\times\!10^{-1}$
      & $4.0\!\times\!10^{-1}$
      & $4.9\!\times\!10^{-1}$
      & $7.3\!\times\!10^{-1}$
      & $1.8$ \\
    PDE-Refiner
      & $3.8\!\times\!10^{-3}$
      & $5.8\!\times\!10^{-3}$
      & $1.1\!\times\!10^{-2}$
      & $3.4\!\times\!10^{-2}$
      & $6.6\!\times\!10^{-2}$
      & $1.1\!\times\!10^{-1}$
      & $2.0\!\times\!10^{-1}$
      & $3.8\!\times\!10^{-1}$ \\
    UNet + comm.\ reg.
      & $1.6\!\times\!10^{-2}$
      & $2.1\!\times\!10^{-2}$
      & $2.1\!\times\!10^{-2}$
      & $2.3\!\times\!10^{-2}$
      & $2.5\!\times\!10^{-2}$
      & $3.6\!\times\!10^{-2}$
      & $4.7\!\times\!10^{-2}$
      & $6.1\!\times\!10^{-2}$ \\
    \bottomrule
  \end{tabular}
\end{table}

The UNet baseline is the more accurate model at the very first
rollout steps ($\le 50$) but is overtaken by both alternatives by step
$\sim$$100$ and is more than an order of magnitude worse than UNet+CR
already inside the training window (step $200$).
PDE-Refiner is the strongest \emph{unregularised} model up to step
$\sim$$1000$, paying $5\times$ inference-time cost; from step
$\sim$$2000$ onwards UNet+CR overtakes it on both the in-distribution
and the out-of-distribution split. The relative gain is largest in the
$25\times$-extrapolation regime, where the unregularised UNet rollout
has lost coherence and PDE-Refiner has accumulated enough error to
also degrade.

\subsection{Full per-horizon results: spectral backbones}
\label{app:exp:kdv:numbers_fno}

\begin{table}[h]
  \centering
  \caption{KdV final-step nMSE for the spectral backbones on the
           in-distribution single-soliton test set
           ($120$ trajectories, mean over three seeds. Training horizon is $200$ steps; the rightmost
           four columns are extrapolation in rollout length.
           Lower is better.}
  \label{tab:kdv_app_fno_full}
  \scriptsize
  \setlength{\tabcolsep}{3pt}
  \begin{tabular}{lccccccc}
    \toprule
    Step & $50$ & $100$ & $200$ & $500$ & $1000$ & $1500$ & $2000$ \\
    \midrule
    FNO
      & $6.3\!\times\!10^{-4}$
      & $7.6\!\times\!10^{-2}$
      & $3.1\!\times\!10^{-1}$
      & $3.9\!\times\!10^{-1}$
      & $4.9\!\times\!10^{-1}$
      & $5.0\!\times\!10^{-1}$
      & $5.0\!\times\!10^{-1}$ \\
    U-FNO
      & $2.1\!\times\!10^{-4}$
      & $6.1\!\times\!10^{-4}$
      & $1.5\!\times\!10^{-3}$
      & $7.2\!\times\!10^{-3}$
      & $1.7\!\times\!10^{-1}$
      & $1.6\!\times\!10^{-1}$
      & $1.7\!\times\!10^{-1}$ \\
    FNO + comm.\ reg.
      & $2.1\!\times\!10^{-4}$
      & $6.1\!\times\!10^{-4}$
      & $1.5\!\times\!10^{-3}$
      & $7.3\!\times\!10^{-3}$
      & $1.9\!\times\!10^{-2}$
      & $3.4\!\times\!10^{-2}$
      & $5.1\!\times\!10^{-2}$ \\
    U-FNO + comm.\ reg.
      & $2.1\!\times\!10^{-4}$
      & $6.2\!\times\!10^{-4}$
      & $1.5\!\times\!10^{-3}$
      & $4.4\!\times\!10^{-3}$
      & $1.3\!\times\!10^{-2}$
      & $2.9\!\times\!10^{-2}$
      & $4.6\!\times\!10^{-2}$ \\
    \bottomrule
  \end{tabular}
\end{table}

\paragraph{Why the unregularised FNO collapses.}
At the end of training the baseline FNO has a one-step validation MSE
of $\sim$$1.5\!\times\!10^{-7}$ and the regularised FNO
$\sim$$6.9\!\times\!10^{-7}$, i.e.\ the baseline is $\sim$$4\times$
\emph{better} per step. Yet the baseline's autoregressive rollout
amplifies error super-exponentially and reaches nMSE $\sim$$0.3$ by
step $200$, while FNO+CR maintains nMSE $\sim$$1.5\!\times\!10^{-3}$
at the same horizon. The U-FNO baseline is more stable than the plain
FNO baseline because the embedded multi-scale UNet path injects local
mixing that breaks the otherwise diagonal Fourier-block structure, but
it still collapses around step $\sim$$1000$. In both cases the
regulariser converts a divergent rollout into a bounded one.

\subsection{In-distribution Per-trajectory snapshots}
\label{app:exp:kdv:snapshots}
See Figure \ref{fig:kdv_snapshots_rewrite}.

\begin{figure}[h]
  \centering
  \includegraphics[width=\linewidth]{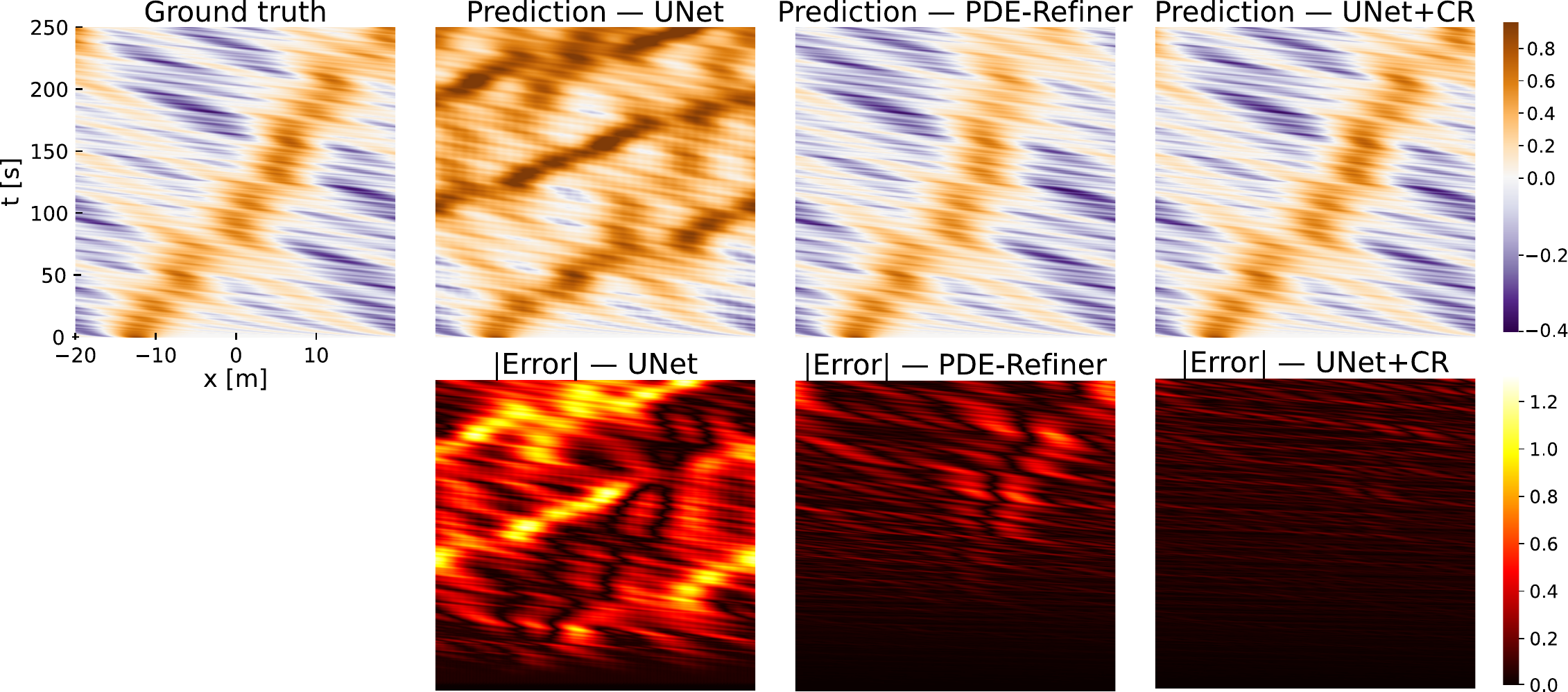}
  \caption{KdV space--time plots of $u(x,t)$ for representative in-distribution test trajectories: ground truth, baseline rollout, and
           commutativity-regularised rollout, run for the full $5000$
           steps (UNet variants) 
           from a single initial condition. Color scale is symmetric
           and shared per trajectory.  }
  \label{fig:kdv_snapshots_rewrite}
\end{figure}

\subsection{Out-of-distribution Per-trajectory snapshots}
\label{app:exp:kdv:snapshots_ood}
See Figure \ref{fig:kdv_snapshots_ood}.

\begin{figure}[h]
  \centering
  \includegraphics[width=\linewidth]{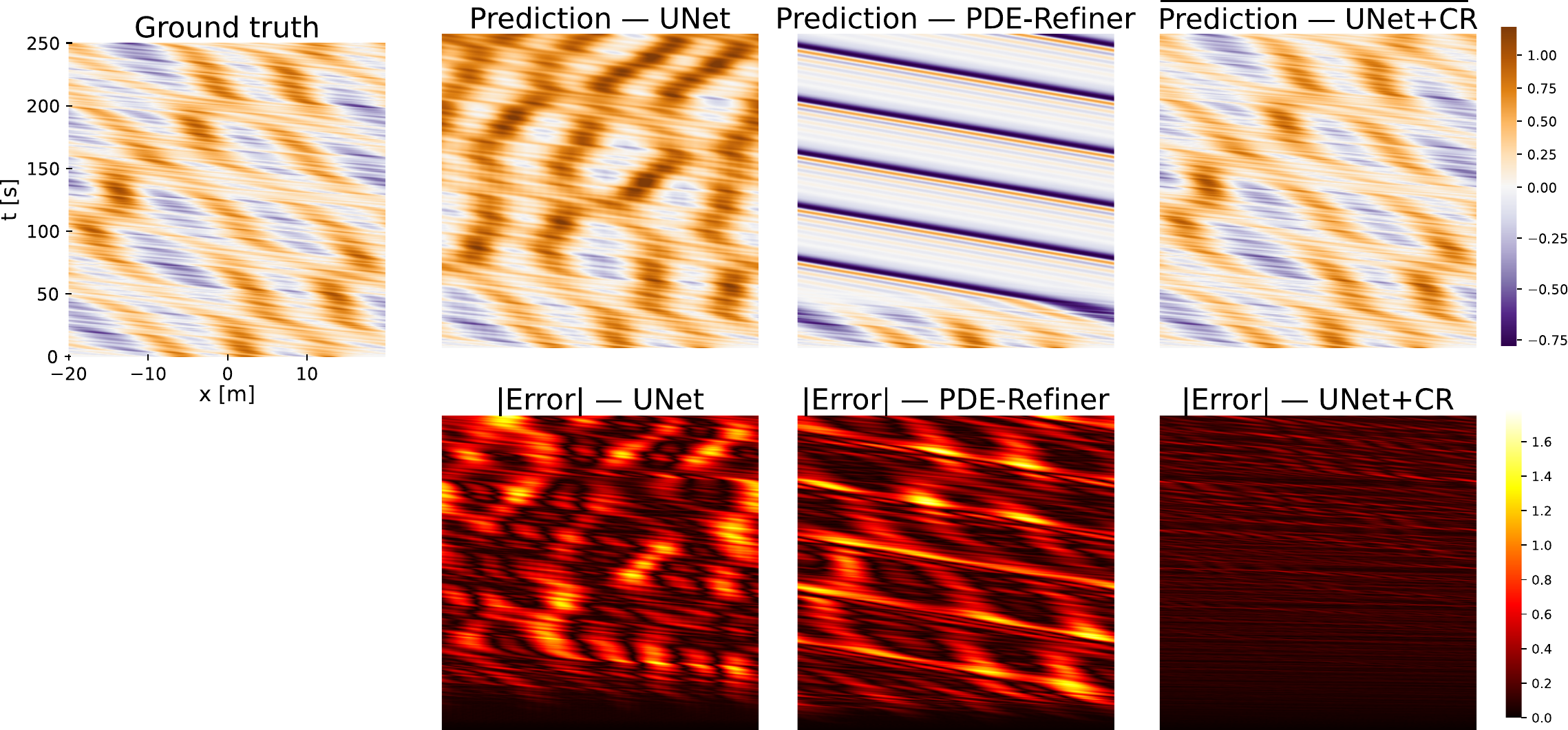}
  \includegraphics[width=\linewidth]{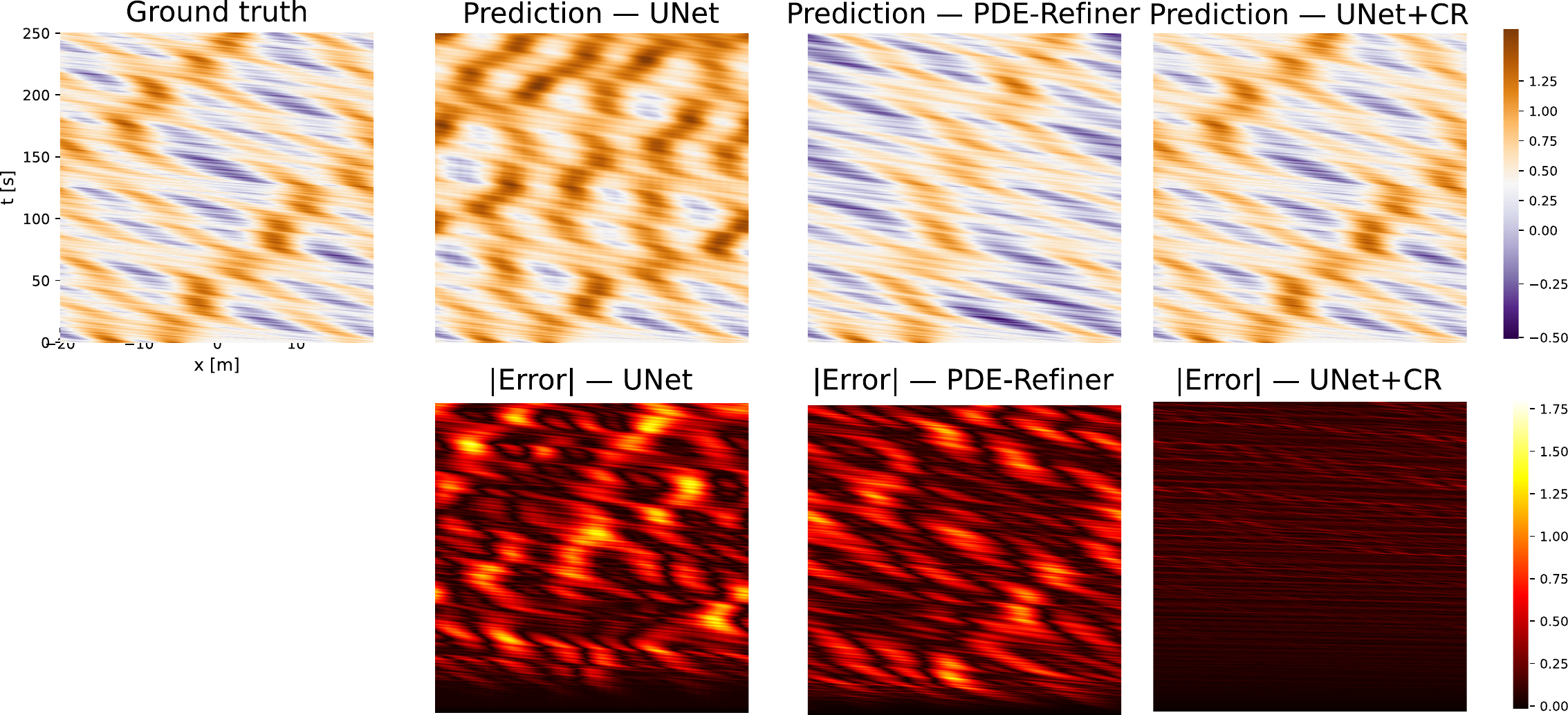}
  \caption{KdV space--time plots of $u(x,t)$ for representative OOD test
           trajectories: ground truth, baseline rollout, and
           commutativity-regularised rollout, run for the full $5000$
           steps (UNet variants).}
  \label{fig:kdv_snapshots_ood}
\end{figure}

\begin{figure}[h]
  \centering
  \includegraphics[width=\linewidth]{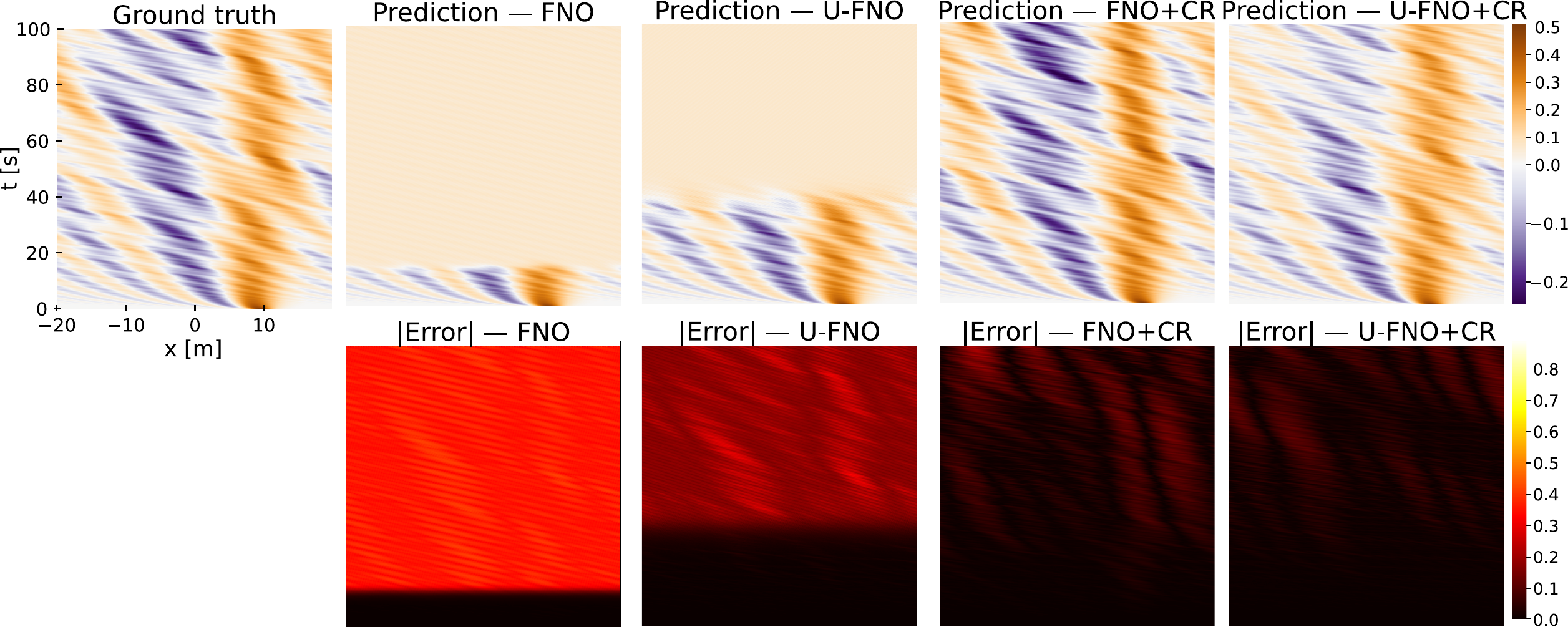}
  \includegraphics[width=\linewidth]{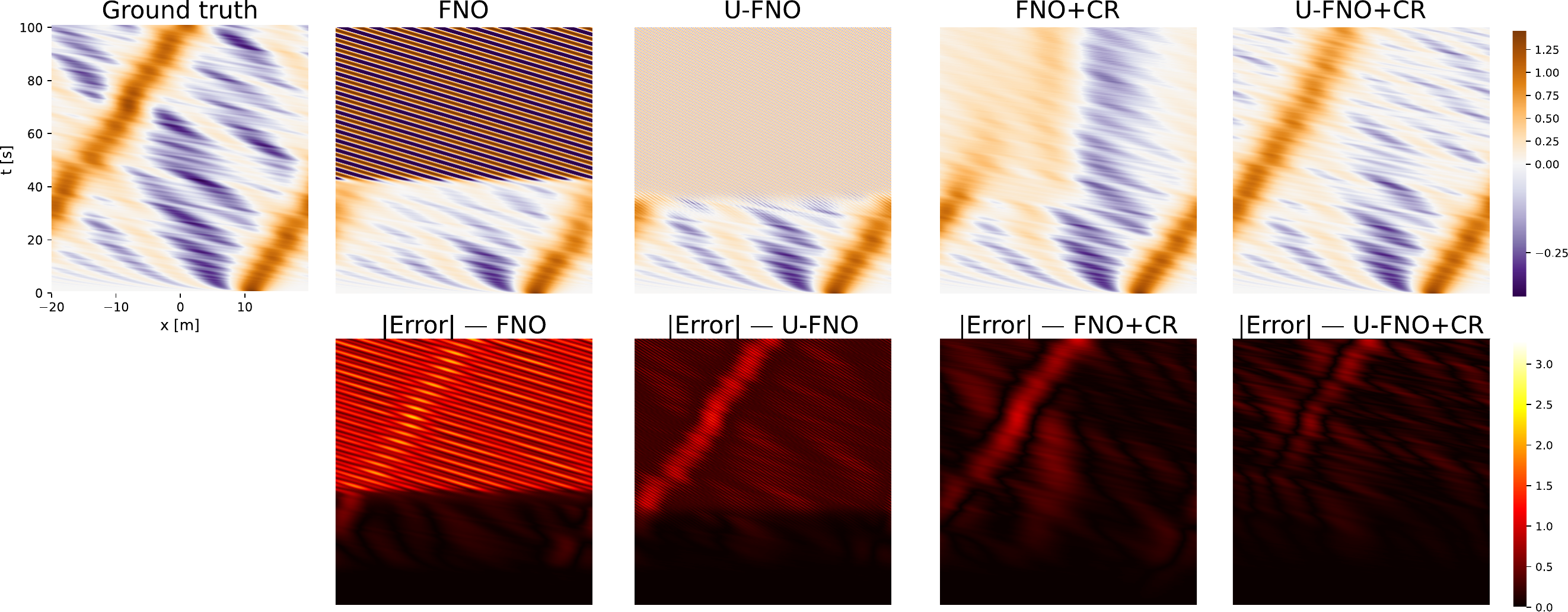}
  \caption{KdV space--time plots of $u(x,t)$ for representative test
           trajectories: ground truth, baseline rollout, and
           commutativity-regularised rollout, run for the full $2000$
           steps (FNO variants).}
  \label{fig:kdv_snapshots_fno}
\end{figure}

\section{BVE experiment: details}
\label{app:exp:bve}

This appendix expands the BVE experiment of Section~\ref{sec:exp:bve}:
data generation, the 2D UNet architecture, training and regulariser
hyperparameters, evaluation, the full
per-horizon RMSE numbers underlying
Figure~\ref{fig:bve_rollout_rmse} and Table~\ref{tab:bve_full_horizons}, and
per-trajectory snapshots.

\subsection{Data generation}
\label{app:exp:bve:data}

\paragraph{Equation and solver.}
We integrate the 2D barotropic vorticity equation
\begin{equation}
  \partial_t \zeta + J(\psi,\,\zeta + \beta y)
   \;=\; -\nu(-\nabla^{2})^{2}\zeta - r\zeta,
   \qquad \zeta = \nabla^{2}\psi,
  \label{eq:bve}
\end{equation}
on the doubly-periodic domain $[0, 2\pi]^{2}$ with $N{=}64$ grid
points per side. The solver is pseudospectral in space with the
$2/3$-rule dealiasing mask and RK4 in time. The Jacobian
$J(\psi,q) = \partial_x\psi\,\partial_y q
            -\partial_y\psi\,\partial_x q$
is evaluated in real space; all linear operations
(Laplacian inversion, derivatives, hyperviscosity) are evaluated
in spectral space. Physical and numerical parameters are summarised
in Table~\ref{tab:bve_phys}.

\begin{table}[h]
  \centering
  \caption{BVE solver parameters.}
  \label{tab:bve_phys}
  \small
  \begin{tabular}{ll}
    \toprule
    Domain                & $[0, 2\pi]^{2}$, doubly periodic \\
    Grid                  & $N \times N = 64 \times 64$ \\
    $\beta$ (planetary vorticity gradient) & $1.0$ \\
    Hyperviscosity        & $\nu = 10^{-8}$, biharmonic order $p=2$ \\
    Ekman drag $r$        & $10^{-2}$ \\
    Time step $\Delta t$  & $5\times 10^{-4}$ \\
    Save stride           & every $100$ steps;
                            $\Delta t_{\text{out}} = 0.05$\,s \\
    Spin-up               & $2.0$\,s (discarded) \\
    Recording window      & $10$\,s ($200$ snapshots/trajectory) \\
    Dealiasing            & $2/3$ rule \\
    Time integrator       & RK4 \\
    \bottomrule
  \end{tabular}
\end{table}

\paragraph{Initial conditions.}
Initial conditions are random Gaussian fields with prescribed
energy spectrum
\begin{equation}
  E(k) \;\propto\; k^{4}\,
  \exp\!\bigl(-2\,(k/k_{\text{peak}})^{2}\bigr),
  \qquad k_{\text{peak}} = 6,
\end{equation}
sampled in spectral space with i.i.d.\ uniform phases on
$[0, 2\pi)$, then rescaled in real space so that the RMS vorticity is
$1.5$. Each trajectory is integrated through a $2$\,s spin-up that is
discarded, after which $200$ snapshots are recorded at
$\Delta t_{\text{out}}{=}0.05$\,s, giving a $10$\,s recording window.

\paragraph{Splits and normalisation.}
The dataset contains $300$ trajectories, split deterministically into
$240$ / $30$ / $30$ for train / val / test.
Vorticity fields are standardised to zero mean and unit standard
deviation using the train-split statistics before being passed
through the network; predictions are evaluated in this normalised
space. All RMSE numbers in the main text and this appendix are in
these normalised vorticity units. 

\subsection{Architecture}
\label{app:exp:bve:arch}

The 2D UNet advances the field one step in time,
$\zeta_t \mapsto \zeta_{t+1}$, with the following layout:
\begin{itemize}
  \item \textbf{Encoder.} Three residual conv blocks with channel
        widths $(64, 128, 256)$. Each block is two
        $3{\times}3$ circular convolutions wrapped in
        \texttt{Conv $\to$ GroupNorm($\min(8,C)$) $\to$ GELU},
        with a $1{\times}1$ skip when the channel count changes.
        Each level is followed by $2{\times}2$ average pooling,
        producing skip activations at resolutions $64, 32, 16$.
  \item \textbf{Bottleneck.} A residual conv block
        $\mathrm{ConvBlock}(256 \to 256)$ at the
        $8 \times 8 \times 256 = 16384$-dimensional latent state.
  \item \textbf{Decoder.} Two transposed
        $2{\times}2$ convolutions interleaved with residual conv
        blocks consuming the corresponding encoder skip activations
        at resolutions $16$ and $32$. A bilinear interpolation
        guards against odd-dimension rounding before concatenation.
  \item \textbf{Head.} A $1{\times}1$ convolution to a single
        output channel.
\end{itemize}
The latent state used by the regulariser is the output of
\texttt{encode($\zeta$)} after all encoder downsamplings and the
bottleneck residual block, and the latent advance operator is
$F(z, \mathrm{skips}) =
\mathrm{encode}\bigl(\mathrm{decode}(z, \mathrm{skips})\bigr)$,
exactly as in the KdV UNet of Section~\ref{sec:exp:kdv}. The Jacobian
penalised by the regulariser is therefore the Jacobian of this
encode$\circ$decode map at the bottleneck, naturally coupling
bottleneck channels with the skip-feature channels.

\subsection{Training and regulariser hyperparameters}
\label{app:exp:bve:hp}

\begin{table}[h]
  \centering
  \caption{BVE training hyperparameters, identical between baseline
           and regularised regimes except for the loss term.}
  \label{tab:bve_hp}
  \small
  \begin{tabular}{ll}
    \toprule
    Optimiser            & AdamW \\
    Peak learning rate   & $10^{-4}$ \\
    Weight decay         & $10^{-5}$ \\
    Schedule             & cosine annealing to $10^{-7}$ \\
    Epochs               & $500$ \\
    Batch size           & $128$ \\
    Training trajectory length & $200$ frames ($10$\,s) \\
    Loss (one-step)      & $\|F_\theta(\zeta_t) - \zeta_{t+1}\|^2$ (MSE) \\
    Checkpoint selection & best validation MSE \\
    \midrule
    Regulariser          & latent-space comm.\ + normality \\
    $\lambda_\mathrm{c}$ & $10^{-7}$ \\
    $\lambda_\mathrm{n}$ & $10^{-7}$ \\
    JVP frequency        & every $15$th minibatch \\
    Regulariser sub-batch& first $25$ samples of the minibatch \\
    Regulariser pair     & in-trajectory adjacent step
                           $(\zeta_{t'},\zeta_{t'+1})$,
                           $t'$ uniform per sample \\
    Probe vector $v$     & i.i.d.\ Gaussian, fresh per evaluation \\
    \bottomrule
  \end{tabular}
\end{table}

\paragraph{Construction of the in-trajectory pair.}
Each minibatch sample carries, in addition to the supervised pair
$(\zeta_t, \zeta_{t+1})$, an independently drawn second adjacent
pair $(\zeta_{t'}, \zeta_{t'+1})$ from the \emph{same} trajectory
with $t'$ uniform over the recorded frames. The regulariser encodes
both endpoints of this partner pair,
$z_{t'} = \mathrm{encode}(\zeta_{t'})$ and
$z_{t'+1} = \mathrm{encode}(\zeta_{t'+1})$, and computes
\begin{equation}
  \mathcal{L}_\mathrm{c}
   \;=\; \bigl\|[J_{t'},\,J_{t'+1}]\,v\bigr\|^{2},
   \qquad
  \mathcal{L}_\mathrm{n}
   \;=\; \bigl\|(J_{t'}^\top J_{t'} - J_{t'} J_{t'}^\top)\,v\bigr\|^{2},
\end{equation}
on the latent state space, with
$J_\bullet = \partial F/\partial z\,|_{(z_\bullet)}$
and $v$ a fresh Gaussian probe vector. Each Jacobian--vector product
is computed by a single \texttt{torch.func.jvp} call (and one
\texttt{vjp} for the normality term). The total loss is
$\mathcal{L} = \mathrm{MSE}
              + \lambda_\mathrm{c}\mathcal{L}_\mathrm{c}
              + \lambda_\mathrm{n}\mathcal{L}_\mathrm{n}$
with $\lambda_\mathrm{c} = \lambda_\mathrm{n} = 10^{-7}$.
The regulariser is applied on the first $25$ samples of every
$15$th minibatch, which keeps the training-time overhead small
relative to the baseline per epoch while still providing a fresh
in-trajectory pair on each application.

\subsection{Evaluation Metric}
\label{app:exp:bve:eval}

The reported metric is the per-step spatial RMSE on the
normalised vorticity field,
\begin{equation*}
  \mathrm{RMSE}(t)
   \;=\;
  \sqrt{\tfrac{1}{N^{2}}\sum_{i,j}(\hat\zeta_{t,ij} - \zeta_{t,ij})^{2}},
\end{equation*}
averaged over the $30$ test trajectories per seed.

\subsection{Full per-horizon results}
\label{app:exp:bve:numbers}

Table~\ref{tab:bve_full_horizons} reports the per-step RMSE at every
$20$ rollout steps on the same $30$ test trajectories, together with
the across-seed standard deviation. Two features are visible across
the entire range. First, the baseline crosses RMSE $1$ around step
$33$ ($t{\approx}1.65$\,s, well within the training horizon) and
saturates near RMSE $\approx 5$--$6$ from step $\sim$$140$ onward;
the regularised model stays bounded throughout. 
The regulariser
improves rollout variance
 suggesting that the unconstrained Jacobian product dominates
the variability of the baseline rollout while the
regularised rollout is essentially a structural property of the data
that all seeds converge to.

\begin{table}[h]
  \centering
  \caption{BVE per-step RMSE on the $30$-trajectory test split at
           every $20$ rollout steps (mean over $3$
           training seeds, normalised vorticity units).
           $\Delta t_{\text{out}} = 0.05$\,s.}
  \label{tab:bve_full_horizons}
  \scriptsize
  \setlength{\tabcolsep}{3pt}
  \begin{tabular}{lcccccccccc}
    \toprule
    Step    & $1$ & $20$ & $40$ & $60$ & $80$ & $100$ & $120$ & $140$ & $160$ & $199$ \\
    Time (s)& $0.05$ & $1$ & $2$ & $3$ & $4$ & $5$ & $6$ & $7$ & $8$ & $9.95$ \\
    \midrule
    Baseline (mean)
        & $0.0040$
        & $0.272$
        & $1.582$
        & $2.810$
        & $3.620$
        & $4.266$
        & $4.890$
        & $5.391$
        & $5.721$
        & $6.007$ \\
    Baseline (std)
        & $0.0001$
        & $0.027$
        & $0.242$
        & $0.381$
        & $0.332$
        & $0.282$
        & $0.695$
        & $1.187$
        & $1.540$
        & $1.896$ \\
    \midrule
    Comm.\ reg.\ (mean)
        & $0.0028$
        & $0.0512$
        & $0.0962$
        & $0.140$
        & $0.186$
        & $0.239$
        & $0.302$
        & $0.377$
        & $0.465$
        & $0.665$ \\
    Comm.\ reg.\ (std)
        & $0.00003$
        & $0.0007$
        & $0.0014$
        & $0.002$
        & $0.003$
        & $0.004$
        & $0.005$
        & $0.006$
        & $0.007$
        & $0.009$ \\
    \midrule
    Ratio (mean)
        & $1.4$ & $5.3$ & $16.4$ & $20.1$ & $19.5$ & $17.8$ & $16.2$ & $14.3$ & $12.3$ & $9.0$ \\
    \bottomrule
  \end{tabular}
\end{table}

\subsection{Per-trajectory snapshots}
\label{app:exp:bve:snapshots}
See Figure \ref{fig:bve_snapshots_rewrite}

\begin{figure}[p]
    \vspace{-1.5cm}
  \centering
  \includegraphics[width=\linewidth]{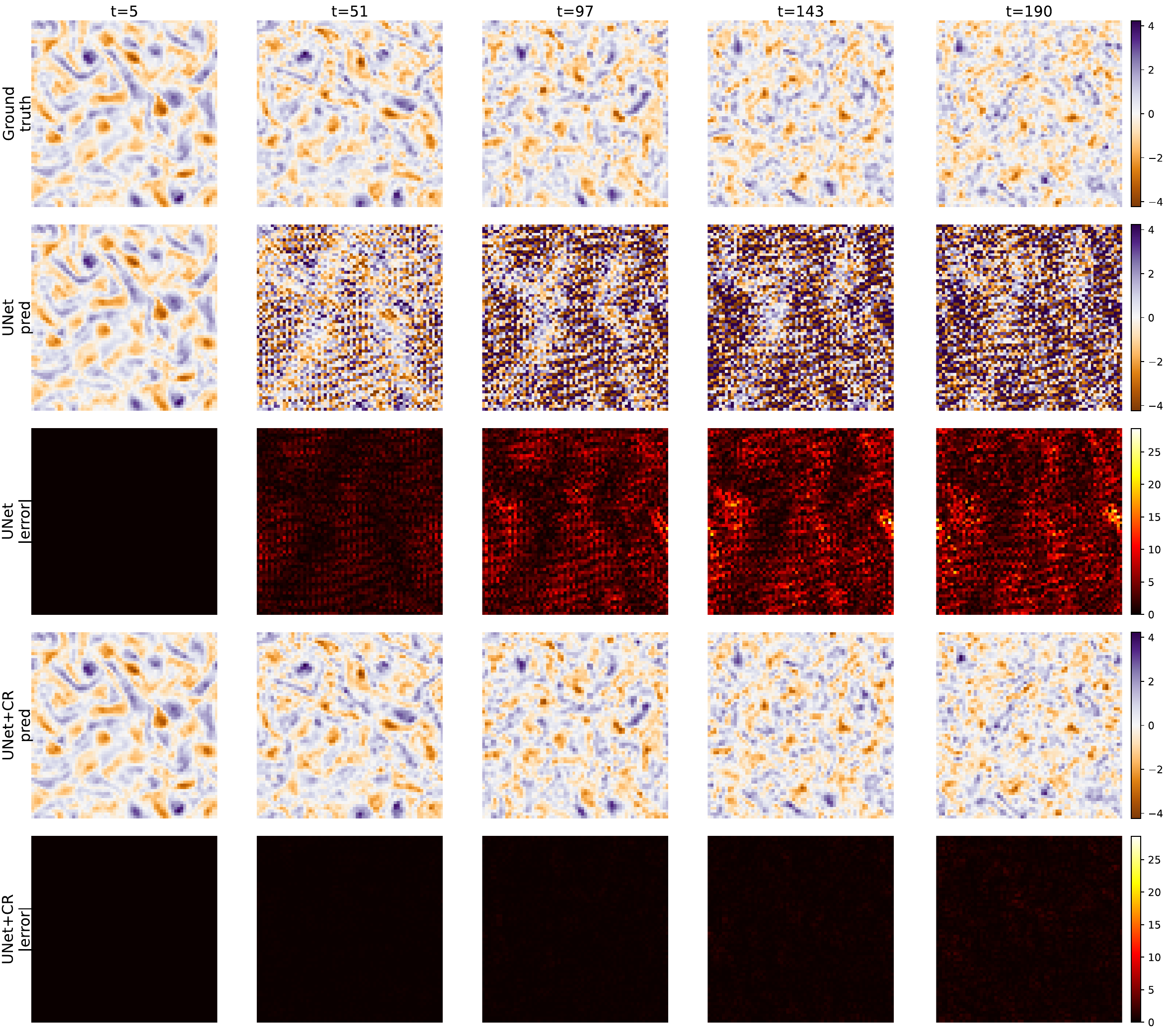}
  
  \caption{Vorticity snapshots $\zeta(x,y,t)$ for representative
           test trajectories: ground truth, baseline rollout, and
           commutativity-regularised rollout, all run for the full
           $199$ rollout steps ($\sim 10$\,s) from a single initial
           condition. 
           Color scale is symmetric and shared between
           truth and predictions per trajectory; absolute error uses
           a separate scale. The baseline visibly amplifies vorticity
           to several times the natural scale of the flow well before
           the end of the training horizon, while the regularised
           model preserves coherent structures and the correct
           amplitude band over the full rollout.
           }
  \label{fig:bve_snapshots_rewrite}
\end{figure}

\begin{figure}[p]
    \vspace{-1.5cm}
  \centering
  
  \includegraphics[width=\linewidth]{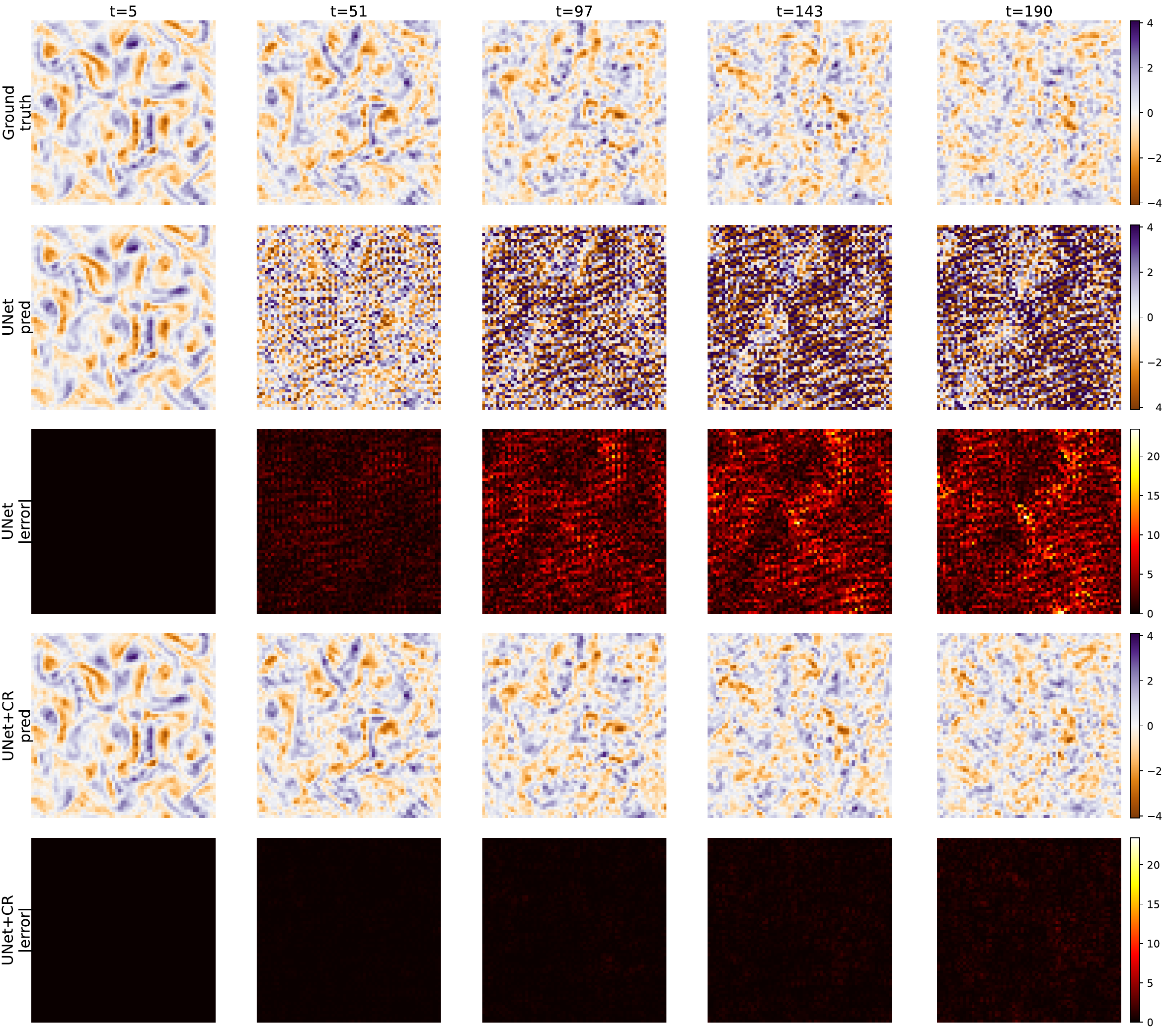}
  \caption{Further vorticity snapshots $\zeta(x,y,t)$ for representative
           test trajectories: ground truth, baseline rollout, and
           commutativity-regularised rollout, all run for the full
           $199$ rollout steps ($\sim 10$\,s) from a single initial
           condition. 
           }
  \label{fig:bve_snapshots_rewrite2}
\end{figure}

\section{FCN/ERA5 experiment: details}
\label{app:exp:fcn}

This appendix expands the FCN experiment of
Section~\ref{sec:exp:fcn}: data and preprocessing, the
training/eval protocol, the regulariser implementation in the AFNO
latent stack, and the full per-channel, per-lead RMSE 
numbers underlying Figure~\ref{fig:fcn_rmse_curves} in 
Tables~\ref{tab:fcn_rmse_full_2018} and \ref{tab:fcn_rmse_full_2019}.

\subsection{Data and preprocessing}
\label{app:exp:fcn:data}

\paragraph{Source.}
ERA5 reanalysis on FCN's native $20$-channel state at
$721{\times}1440$ ($0.25^{\circ}$, $\Delta t = 6$\,h):
$10$\,m winds $(u_{10}, v_{10})$, $2$\,m temperature
$t_{2\mathrm{m}}$, surface and mean-sea-level pressures
$(\mathrm{sp}, \mathrm{msl})$, $T$ at $850$\,hPa, winds at
$1000$/$850$/$500$\,hPa, geopotential at $1000$/$850$/$500$/$50$\,hPa,
relative humidity at $500$/$850$\,hPa, $T$ at $500$\,hPa, and total
column water vapour $\mathrm{tcwv}$. We use FCN's per-variable
normalisation statistics throughout, so the network is run in its
native input space; metric values are reported in physical units
after the inverse normalisation.

\paragraph{Years.}
Pretraining of FCN v1 covers $1979$--$2015$, so all
fine-tuning years used here lie inside its training range and
expose the model to no new physical signal:
\begin{itemize}\itemsep -0.2em
  \item FCN+CR\,[2013--15]: fine-tuned on
        $\{2013, 2014, 2015\}$.
  \item Plain-finetune
        baseline: fine-tuned on $\{2013, 2014, 2015\}$.
\end{itemize}
Validation uses 2017 during fine-tuning. The reported test years
are $2018$ (Section~\ref{sec:exp:fcn}) and $2019$ (held-out
replication, Appendix~\ref{app:exp:fcn:numbers}); both are
genuinely outside FCN v1's training range.

\paragraph{Evaluation initial conditions.}
For each test year we sample $41$ initial conditions at a fixed
stride (every $9$\,h, stride=$36$ on the
$6$\,h steps) and run an autoregressive $40$-step rollout
($10$ days) per IC. The reported metric is the mean over ICs of the
latitude-weighted RMSE,
\begin{equation*}
  \mathrm{RMSE}(t)
  \;=\;
  \sqrt{\,
    \tfrac{1}{HW}\sum_{i,j}
      w_i\,\bigl(\hat X_{t,ij} - X_{t,ij}\bigr)^{2}\,},
  \qquad
  w_i = \cos(\phi_i),
\end{equation*}
on each of the channels
$z_{500}$, $t_{850}$, $u_{850}$, $t_{2\mathrm{m}}$. 
Numbers are reported in physical units
after inverse-normalising with the FCN per-channel statistics.

\subsection{Architecture and where the regulariser acts}
\label{app:exp:fcn:arch}

The backbone is the released FCN v1 AFNO with $74.7$\,M total
parameters. Fine-tuning unfreezes the $34.5$\,M parameters of the
deep AFNO blocks and the head; the patch-embed and the first
several blocks are kept frozen and run once per minibatch with their
output detached.

\paragraph{Latent space operator.}
The latent state on which the regulariser acts is the activation
\emph{after} the first $K{=}10$ AFNO blocks; the regulariser
flows JVP / VJP only through the remaining (deep) blocks, so the
Jacobian--vector products it touches have memory cost
$\propto\!(\text{depth}-K)$ rather than full depth. This is
necessary at $720{\times}1440$ resolution: fully differentiating the
network would otherwise blow GPU memory at batch size $\geq 1$, even
at low precision. 

\subsection{Training and regulariser hyperparameters}
\label{app:exp:fcn:hp}

\begin{table}[h]
  \centering
  \caption{FCN fine-tuning hyperparameters (identical across the
           two fine-tune regimes except for the loss term).}
  \label{tab:fcn_hp}
  \small
  \begin{tabular}{ll}
    \toprule
    Backbone               & FCN v1 AFNO, pretrained
                              \citep{pathak2022fourcastnet} \\
    Optimiser              & AdamW \\
    Peak learning rate     & $5\times 10^{-6}$ \\
    Weight decay           & $0$ \\
    Schedule               & cosine annealing \\
    Epochs                 & $50$ \\
    Batch size             & $2$ (single-GPU, A100\,80\,GB) \\
    Loss (one-step)        & latitude-weighted MSE
                              between $\hat X_{t+1}$ and ERA5 \\
    \midrule
    Regulariser            & latent comm.\ + normality \\
    $\lambda_\mathrm{c}$   & $10^{-5}$ \\
    $\lambda_\mathrm{n}$   & $10^{-5}$ \\
    JVP frequency          & every minibatch
                              (\texttt{comm\_freq}{=}$1$) \\
    Skip blocks            & first $10$ AFNO blocks detached \\
    Comm.\ pair            & adjacent-step pair $(X_t, X_{t+1})$\\
    Probe vector $v$       & i.i.d.\ Rademacher\\ %
    \bottomrule
  \end{tabular}
\end{table}

\paragraph{Construction of the JVP.}
At each (regulariser-active) minibatch we draw the supervised pair
$(X_t, X_{t+1})$ and compute, in the post-skip latent space,
\begin{equation*}
  \mathcal{L}_\mathrm{c}
   \;=\; \bigl\|[J(z_t),\,J(z_{t+1})]\,v\bigr\|^{2},
   \qquad
  \mathcal{L}_\mathrm{n}
   \;=\; \bigl\|J(z_t)^{\top} J(z_t)\,v
                - J(z_t) J(z_t)^{\top}\,v\bigr\|^{2},
\end{equation*}
where $J$ is the Jacobian of the deep-block stack at the indicated
latent state, $v$ is a fresh Rademacher probe, and the JVP/VJP are
evaluated on the first two samples of the minibatch (memory bound).
Three JVPs and one VJP are needed per call; $J(z_t)\,v$ is reused
between the commutator and normality terms. The total fine-tune
loss is
$\mathcal{L} = \mathcal{L}_\mathrm{MSE}
              + \lambda_\mathrm{c}\,\mathcal{L}_\mathrm{c}
              + \lambda_\mathrm{n}\,\mathcal{L}_\mathrm{n}$
with $\lambda_\mathrm{c} = \lambda_\mathrm{n} = 10^{-5}$.

\subsection{Full per-channel, per-lead RMSE numbers}
\label{app:exp:fcn:numbers}

Table~\ref{tab:fcn_rmse_full_2018} (test year $2018$) and
Table~\ref{tab:fcn_rmse_full_2019} (held-out year $2019$) give the
latitude-weighted RMSE in physical units at every day from $1$ to
$10$ on all four WB2 headline channels.

\begin{table}[h]
  \centering
  \caption{ERA5~$2018$ rollout RMSE (latitude-weighted, physical
           units) at daily lead times for the four variables, averaged over $41$ initial conditions
           ($9$\,h stride). \textbf{Bold} marks the best entry
           per column. The FCN+CR entries here use the
           [2013--15] fine-tune.}
  \label{tab:fcn_rmse_full_2018}
  \scriptsize
  \setlength{\tabcolsep}{2.5pt}
  \begin{tabular}{llcccccccccc}
    \toprule
    Channel & Method & 1\,d & 2\,d & 3\,d & 4\,d & 5\,d & 6\,d & 7\,d & 8\,d & 9\,d & 10\,d \\
    \midrule
    \multirow{3}{*}{$z_{500}$ [m$^{2}$/s$^{2}$]}
       & Frozen        & $86.9$ & $165.0$ & $258.3$ & $374.5$ & $494.5$ & $608.8$ & $\mathbf{704.4}$ & $785.0$ & $\mathbf{854.3}$ & $\mathbf{903.7}$ \\
       & Finetune      & $\mathbf{81.7}$ & $\mathbf{162.7}$ & $260.8$ & $381.1$ & $503.0$ & $615.5$ & $710.9$ & $792.9$ & $871.7$ & $930.0$ \\
       & FCN+CR [2013-2015] & $82.7$ & $164.2$ & $264.0$ & $385.0$ & $510.3$ & $622.6$ & $712.5$ & $\mathbf{790.0}$ & $855.3$ & $908.4$ \\
    \midrule
    \multirow{3}{*}{$t_{850}$ [K]}
       & Frozen        & ${0.86}$ & $\mathbf{1.19}$ & $\mathbf{1.59}$ & $\mathbf{2.06}$ & $\mathbf{2.52}$ & $\mathbf{2.98}$ & $\mathbf{3.36}$ & $\mathbf{3.61}$ & $3.88$ & $4.26$ \\
       & Finetune      & $\mathbf{0.84}$ & $\mathbf{1.19}$ & $1.62$ & $2.10$ & $2.57$ & $3.02$ & $3.38$ & $3.67$ & $4.10$ & $5.18$ \\
       & FCN+CR [2013-2015] & ${0.86}$ & $1.22$ & $1.65$ & $2.13$ & $2.60$ & $3.07$ & $3.42$ & $3.68$ & $\mathbf{3.91}$ & $\mathbf{4.14}$ \\
    \midrule
    \multirow{3}{*}{$u_{850}$ [m/s]}
       & Frozen        & $\mathbf{1.60}$ & $\mathbf{2.43}$ & $\mathbf{3.29}$ & $\mathbf{4.15}$ & $\mathbf{4.93}$ & $\mathbf{5.60}$ & $\mathbf{6.06}$ & $\mathbf{6.46}$ & $\mathbf{6.80}$ & $\mathbf{7.01}$ \\
       & Finetune      & $1.61$ & $2.45$ & $3.34$ & $4.23$ & $5.01$ & $5.67$ & $6.15$ & $6.53$ & $6.88$ & $7.20$ \\
       & FCN+CR [2013-2015] & $1.62$ & $2.48$ & $3.39$ & $4.28$ & $5.05$ & $5.73$ & $6.21$ & $6.56$ & $6.90$ & $7.08$ \\
    \midrule
    \multirow{3}{*}{$t_{2\mathrm{m}}$ [K]}
       & Frozen        & $0.97$ & $1.18$ & $\mathbf{1.42}$ & $\mathbf{1.72}$ & $\mathbf{2.03}$ & $\mathbf{2.34}$ & $2.56$ & $3.00$ & $4.32$ & $7.53$ \\
       & Finetune      & $\mathbf{0.95}$ & $\mathbf{1.18}$ & $1.43$ & $1.74$ & $2.05$ & $2.36$ & $2.66$ & $3.97$ & $7.62$ & $16.16$ \\
       & FCN+CR [2013-2015] & $0.97$ & $1.21$ & $1.47$ & $1.79$ & $2.10$ & $2.39$ & $\mathbf{2.62}$ & $\mathbf{2.87}$ & $\mathbf{3.19}$ & $\mathbf{4.48}$ \\
    \bottomrule
  \end{tabular}
\end{table}

\begin{table}[h]
  \centering
  \caption{ERA5~$2019$ replication: rollout RMSE (latitude-weighted,
           physical units) at daily lead times. 
           \textbf{Bold} marks the best entry per column.
           }
  \label{tab:fcn_rmse_full_2019}
  \scriptsize
  \setlength{\tabcolsep}{2.5pt}
  \begin{tabular}{llcccccccccc}
    \toprule
    Channel & Method & 1\,d & 2\,d & 3\,d & 4\,d & 5\,d & 6\,d & 7\,d & 8\,d & 9\,d & 10\,d \\
    \midrule
    \multirow{3}{*}{$z_{500}$ [m$^{2}$/s$^{2}$]}
       & Frozen                & $87.9$ & $\mathbf{166.1}$ & $\mathbf{263.2}$ & $\mathbf{369.9}$ & $\mathbf{481.9}$ & $\mathbf{595.3}$ & $\mathbf{697.7}$ & $\mathbf{788.2}$ & $\mathbf{861.0}$ & $\mathbf{908.9}$ \\
       & Finetune\,[2013--15]  & $\mathbf{84.1}$ & $164.9$ & $263.4$ & $371.2$ & $486.9$ & $602.8$ & $710.5$ & $797.0$ & $868.2$ & $920.1$ \\
       & FCN+CR\,[2013--15]    & ${84.7}$ & $166.2$ & $267.0$ & $377.7$ & $491.3$ & $607.0$ & $714.6$ & $803.6$ & $872.6$ & $920.0$ \\
    \midrule
    \multirow{3}{*}{$t_{850}$ [K]}
       & Frozen                & $\mathbf{0.87}$ & $\mathbf{1.20}$ & $\mathbf{1.62}$ & $\mathbf{2.05}$ & $\mathbf{2.49}$ & $\mathbf{2.91}$ & $\mathbf{3.28}$ & $\mathbf{3.65}$ & $\mathbf{3.93}$ & $\mathbf{4.23}$ \\
       & Finetune\,[2013--15]  & $0.86$ & $1.21$ & $1.63$ & $2.07$ & $2.52$ & $2.94$ & $3.32$ & $3.67$ & $4.10$ & $5.02$ \\
       & FCN+CR\,[2013--15]    & $\mathbf{0.87}$ & $1.23$ & $1.66$ & $2.11$ & $2.57$ & $2.99$ & $3.37$ & $3.73$ & $4.00$ & ${4.24}$ \\
    \midrule
    \multirow{3}{*}{$u_{850}$ [m/s]}
       & Frozen                & $\mathbf{1.62}$ & $\mathbf{2.43}$ & $\mathbf{3.35}$ & $\mathbf{4.15}$ & $\mathbf{4.90}$ & $\mathbf{5.55}$ & $\mathbf{6.03}$ & $\mathbf{6.46}$ & $\mathbf{6.78}$ & $\mathbf{6.95}$ \\
       & Finetune\,[2013--15]  & $1.63$ & $2.45$ & $3.38$ & $4.19$ & $4.95$ & $5.62$ & $6.16$ & $6.53$ & $6.84$ & $7.10$ \\
       & FCN+CR\,[2013--15]    & $1.64$ & $2.48$ & $3.41$ & $4.23$ & $4.98$ & $5.65$ & $6.21$ & $6.62$ & $6.91$ & $7.12$ \\
    \midrule
    \multirow{3}{*}{$t_{2\mathrm{m}}$ [K]}
       & Frozen                & ${0.97}$ & $\mathbf{1.18}$ & $\mathbf{1.42}$ & $\mathbf{1.69}$ & $\mathbf{1.99}$ & $\mathbf{2.28}$ & $\mathbf{2.55}$ & $\mathbf{2.81}$ & $3.58$ & $5.54$ \\
       & Finetune\,[2013--15]  & $\mathbf{0.95}$ & $1.18$ & $1.43$ & $1.71$ & $2.01$ & $2.28$ & $2.61$ & $3.73$ & $6.96$ & $14.43$ \\
       & FCN+CR\,[2013--15]    & ${0.97}$ & $1.20$ & $1.47$ & $1.75$ & $2.06$ & $2.34$ & $2.63$ & ${2.87}$ & $\mathbf{3.33}$ & $\mathbf{3.99}$ \\
    \bottomrule
  \end{tabular}
\end{table}

\subsection{Qualitative forecasts}
\label{app:exp:fcn:qualitative}
See Figure \ref{fig:fcn_forecast_sample} for ERA5 t2m forecasts.

\begin{figure}[h!]
  \centering
  \includegraphics[width=\linewidth]{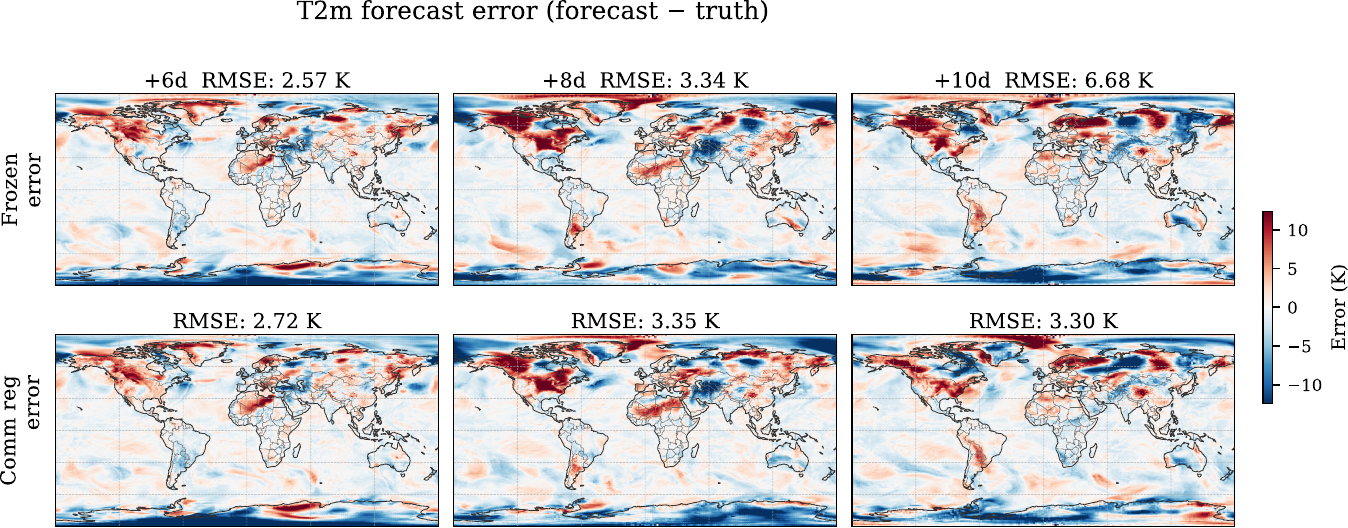}
  \includegraphics[width=\linewidth]{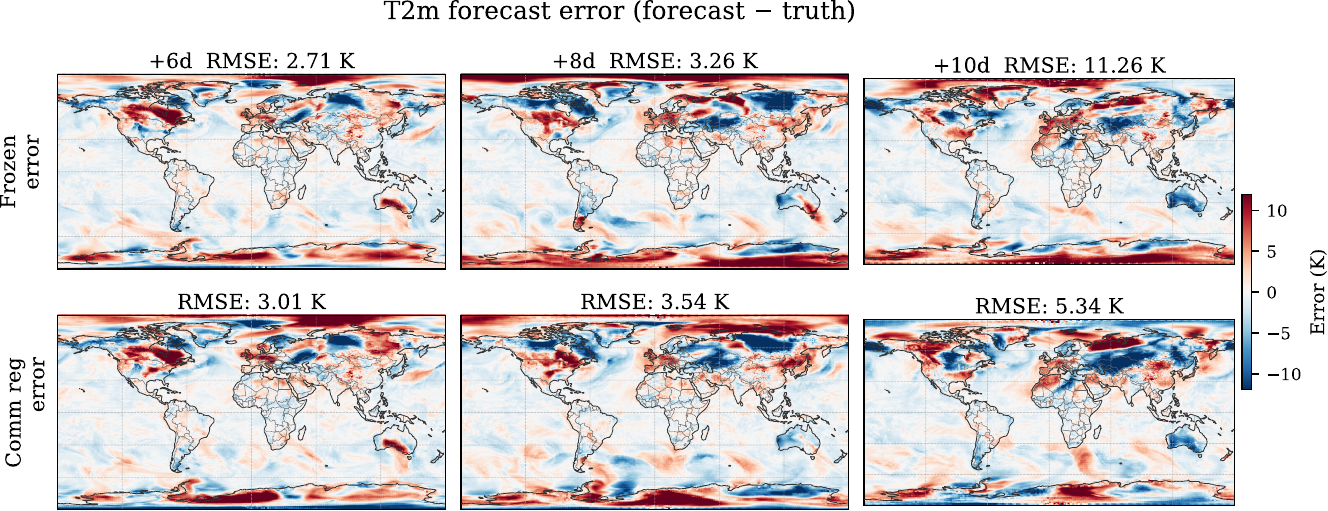}
  \caption{Qualitative day-$6$, $8$ and $10$ rollout error snapshots of
           $t_{2\mathrm{m}}$ on a single ERA5~$2018$ initial
           condition for 2m temperature.}
  \label{fig:fcn_forecast_sample}
\end{figure}

\section{SST: details}
\label{app:exp:sst}

This appendix expands the SST experiment summarised in
Section~\ref{sec:exp:sst}: data and preprocessing, the full UNet
architecture, training hyperparameters, the regulariser
implementation, and the per-step rollout numbers and qualitative
snapshots underlying Figure~\ref{fig:sst:rollout-rmse}.

\subsection{Data and preprocessing}
\label{app:exp:sst:data}

\paragraph{Source.}
NOAA Optimum Interpolation Sea-Surface Temperature, weekly-mean
product (\texttt{sst.wkmean.1990-present})~\citep{reynolds2002oisst},
covering the global ocean on a $1^\circ$ ($180\!\times\!360$) grid
at weekly cadence. The version we use contains $1727$ weeks
beginning in 1990. Land pixels are filled with $0$ in the raw
NetCDF and are masked only in the visualisations
(Appendix~\ref{app:exp:sst:snapshots}); the network sees them as
zero pixels.

\paragraph{Padding.}
The network requires spatial dimensions divisible by $2^4=16$.
We pad each frame from $180\!\times\!360$ to $192\!\times\!384$
using boundary-aware padding: $6$ rows of \emph{reflect} padding
at each pole (no real wrap-around at the poles), and $12$ columns
of \emph{wrap} (periodic) padding at the
$0^\circ$/$360^\circ$ longitude seam. This avoids the mid-Pacific
padding artefacts that a default zero or reflect padding would
introduce and matches the natural periodicity of the longitude axis.

\paragraph{Splits and normalisation.}
We split chronologically into
$1200$ / $150$ / $377$ weeks of train / val / test.
The test split therefore covers $\approx 7.24$\,years and includes
$\sim\!7$ full annual cycles. Frames are standardised to zero mean
and unit standard deviation using train-set statistics; all RMSE
numbers in Section~\ref{sec:exp:sst} and below are reported in
these normalised units.

\subsection{Architecture}
\label{app:exp:sst:arch}

The backbone is a 2D UNet with circular-aware
inputs (the wrap padding above provides the longitude periodicity)
operating on the full $192\!\times\!384$ field.

\paragraph{Encoder/decoder structure.}
Depth $D=4$, base channel width $C_0=64$, so the encoder produces
channel widths $(64, 128, 256, 512)$ at successive levels. Each
encoder block is a residual two-conv block
\texttt{Conv$3\!\times\!3$ $\to$ GroupNorm($\min(8,C)$) $\to$ GELU}
applied twice with a $1\!\times\!1$ skip when the channel count
changes, followed by $2\!\times\!2$ average pooling. The decoder
mirrors the encoder with $2\!\times\!2$ transposed convolutions and
a residual block taking the concatenated skip features. The
bottleneck spatial size is therefore
$192/16\,\times\,384/16 = 12\!\times\!24$ at $512$ channels, the
latent state on which the regulariser acts. A final $1\!\times\!1$
convolution maps back to one output channel (the next-week SST).
The model has $\sim\!17$\,M parameters.

\paragraph{Inputs and outputs.}
$k_{\text{in}}=1$: the network sees one previous week and predicts
the next. We did not find a benefit from longer context windows on
this data.

\subsection{Training and regulariser}
\label{app:exp:sst:hp}

\paragraph{Optimiser and schedule.}
AdamW with learning rate $10^{-4}$, weight decay $10^{-5}$,
batch size $4$, $200$ epochs. The learning rate follows a cosine
annealing schedule with $\eta_{\min} = 10^{-7}$. Loss is one-step
MSE on the next-week SST: the model never sees a multi-step rollout
during training.

\paragraph{Commutativity regulariser.}
The regulariser is identical in form to the BVE one
(Appendix~\ref{app:exp:bve:hp}) but applied to the
$12\!\times\!24$ bottleneck of the SST UNet. At minibatch
$b$ with input $x$ and an in-trajectory random partner $\tilde x$,
the latent advance operator
$F(z, s) = \mathrm{encode}(\mathrm{decode}(z, s))$
is differentiated through \texttt{torch.func.jvp} on the bottleneck
features and skip stack, and the squared norms
\begin{align}
  \mathcal{L}_\mathrm{c}(x,\tilde x)
  \;=\;& \bigl\|[J_F(z(x)),\,J_F(z(\tilde x))]\,v\bigr\|^2,
  \qquad\\
  \mathcal{L}_\mathrm{n}(x)
  \;=\;& \bigl\|J_F(z(x))^\top J_F(z(x))\,v
                - J_F(z(x)) J_F(z(x))^\top v \bigr\|^2,
\end{align}
are estimated on a single random Gaussian probe $v$.
The total loss is
$\mathcal{L} = \mathrm{MSE} + \lambda_\mathrm{c}\mathcal{L}_\mathrm{c}
                              + \lambda_\mathrm{n}\mathcal{L}_\mathrm{n}$
with $\lambda_\mathrm{c} = \lambda_\mathrm{n} = 10^{-7}$, evaluated
every fifth minibatch and on the first $5$ samples of that
minibatch (the exact JVP at full batch and full $192\!\times\!384$
resolution is GPU-memory bound).

\subsection{Evaluation protocol}
\label{app:exp:sst:eval}

We evaluate a single autoregressive rollout starting from the first
frame of the test split and continuing for $T_\mathrm{eval} = 376$
weeks (the test length minus the input context). At each step the
network produces the next normalised SST field, which is then fed
back as the next input. The reported metric is the per-step spatial
RMSE
\[
  \mathrm{RMSE}(t)
  \;=\;
  \sqrt{\tfrac{1}{HW}\sum_{i,j}\bigl(\hat T_{t,ij} - T_{t,ij}\bigr)^2},
\]
on the normalised field, computed over the entire padded
$192\!\times\!384$ grid (including the masked land pixels, whose
target value is $0$ in the dataset). Reporting on the
ocean-pixels-only mask gives values that are uniformly $\sim 5\%$
larger but does not change the relative comparison.

\subsection{Per-step RMSE numbers}
\label{app:exp:sst:numbers}

Table~\ref{tab:sst_rmse_full} reports per-step RMSE at a denser set
of lead times than Section~\ref{sec:exp:sst}. 
The qualitative behaviour, baseline error oscillating with the
annual cycle and growing in envelope, regularised error
phase-locked and bounded, is visible at every horizon beyond
$\sim\!3$\,months.

\begin{table}[t]
  \centering
  \caption{SST rollout spatial RMSE on the normalised test field
           at selected rollout steps (single rollout, $376$ weeks).
           Lower is better; \textbf{bold} marks the better entry per
           column. The two final columns are summary statistics over
           the entire $376$-step rollout.}
  \label{tab:sst_rmse_full}
  \small
  \setlength{\tabcolsep}{4pt}
  \begin{tabular}{lcccccccc|cc}
    \toprule
    Method
        & 1 & 13 & 26 & 52 & 104 & 156 & 260 & 376
        & mean & max \\
    Lead
        & 1\,wk & 3\,mo & 6\,mo & 1\,yr & 2\,yr & 3\,yr & 5\,yr & 7.2\,yr
        & all & all \\
    \midrule
    Baseline
        & $\mathbf{0.025}$
        & $0.059$
        & $0.120$
        & $0.160$
        & $0.250$
        & $0.277$
        & $0.193$
        & $0.148$
        & $0.248$
        & $0.516$ \\
    Commutativity reg.
        & $\mathbf{0.025}$
        & $\mathbf{0.046}$
        & $\mathbf{0.050}$
        & $\mathbf{0.062}$
        & $\mathbf{0.055}$
        & $\mathbf{0.066}$
        & $\mathbf{0.097}$
        & $\mathbf{0.114}$
        & $\mathbf{0.082}$
        & $\mathbf{0.162}$ \\
    \bottomrule
  \end{tabular}
\end{table}

\subsection{Qualitative snapshots}
\label{app:exp:sst:snapshots}
Visualizations of sea surface temperature predictions and errors are shown in Figure \ref{fig:sst_snapshots}.

\begin{figure}[ht!]
  \centering
  \includegraphics[width=\linewidth]{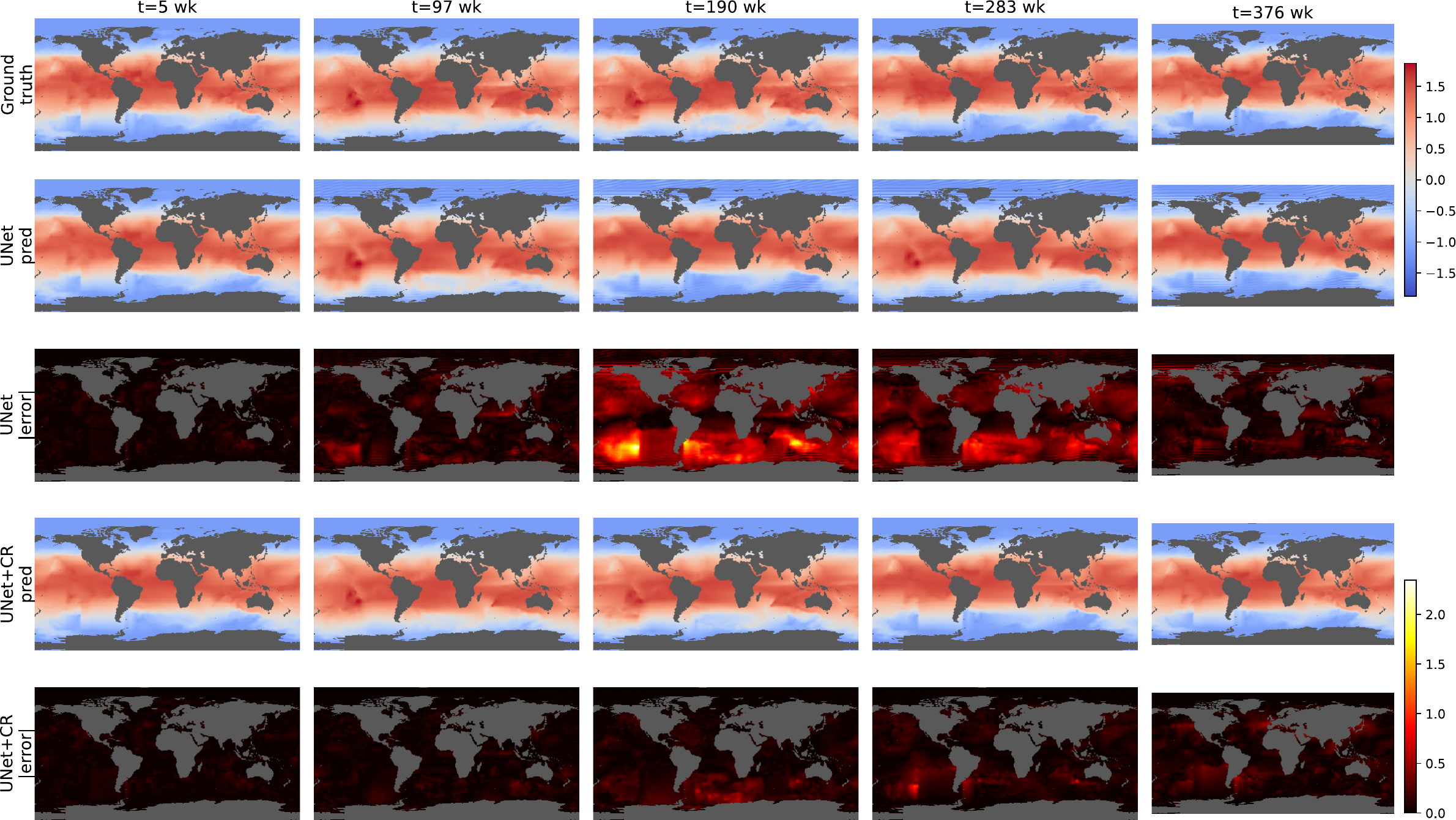}
  \caption{Spatial snapshots of the SST autoregressive rollout at
           selected lead times for the baseline and the
           commutativity-regularised UNet, with ground-truth SST in
           the top row and absolute error in the lower rows. Land is
           overlaid in grey. The baseline progressively blurs
           mid-latitude anomalies and drifts in the seasonal phase;
           the regularised rollout retains the large-scale structure
           of the true field through the full $7.2$-year window.}
  \label{fig:sst_snapshots}
\end{figure}

\end{document}